\documentclass[table]{article}

 \usepackage[preprint]{neurips_2026}

\usepackage[utf8]{inputenc} 
\usepackage[T1]{fontenc}    
\usepackage{hyperref}       
\usepackage{url}            
\usepackage{booktabs}       
\usepackage{amsfonts,amsmath}       
\usepackage{nicefrac}       
\usepackage{microtype}      
\usepackage{xcolor}         
\usepackage{graphicx} 
\usepackage{amsthm}
\usepackage{tikz}
\usepackage{xcolor}
\usepackage{todonotes}
\usetikzlibrary{arrows.meta, positioning, calc}

\usepackage{booktabs}
\usepackage{tabularx}
\usepackage{array}
\usepackage{xcolor}
\usepackage{amssymb}
\usepackage[normalem]{ulem} 
\renewcommand\arraystretch{1.4}

\newtheorem{example}{Example}

\newtheorem{definition}{Definition}

\newtheorem{remark}{Remark}
\newtheorem{corollary}{Corollary}
\newtheorem{proposition}{Proposition}
\newtheorem{desideratum}{Desideratum}

\definecolor{layerbg}{RGB}{238,237,254}
\definecolor{infobg}{RGB}{241,239,232}
\definecolor{exbg}{RGB}{225,245,238}
\definecolor{boundbg}{RGB}{250,238,218}

\newcommand{\xiaowei}[1]{{\color{blue}#1}}
\newcommand{\changshun}[1]{{\color{red}#1}}
\newcommand{\ayoub}[1]{{\color{purple}#1}}
\newcommand{\raj}[1]{\color{teal}#1}
\newcommand{\dejan}[1]{{\color{brown}#1}}

\newcommand{\martinComment}[1]{\todo[inline,color=yellow]{Martin: #1}}
\renewcommand{\martinComment}[1]{}
\usepackage{comment}

\title{Position: A Three-Layer Probabilistic Assume–Guarantee Architecture
Is Structurally Required for Safe LLM Agent Deployment}

\author{%
  Saddek Bensalem$^{1}$
  \quad Yi Dong$^{2}$
  \quad Martin Fr\"anzle$^{3}$
  \quad Xiaowei Huang$^{2}$
  \quad Janis Kr\"oger$^{3}$ \\[2pt]
  \bf Dejan Nickovic$^{4}$
  \quad Ayoub Nouri$^{5}$
  \quad Rajarshi Roy$^{2}$
  \quad Changshun Wu$^{5}$ \\[6pt]
  \normalfont\normalsize
  $^{1}$CSX-AI, France \quad
  $^{2}$University of Liverpool, UK \quad
  $^{3}$Carl von Ossietzky Universit\"at Oldenburg, Germany \\
  $^{4}$Austrian Institute of Technology, Austria \quad
  $^{5}$Universit\'e Grenoble Alpes, France
}

\begin{document}

\maketitle

\begin{abstract}
  This position paper argues that enforcing LLM agent safety within a single abstraction layer is not merely suboptimal but categorically insufficient for deployed LLM agents---a structural consequence of how agent execution works, not a contingent limitation of current systems. The three dimensions that jointly constitute safe operation---semantic intent and policy compliance, environmental validity, and dynamical feasibility---each depend on a strictly distinct information set that becomes available at a different stage of execution. No single guardrail can certify all three. 
We argue the community must respond with a contract-based architecture in which each safety dimension is enforced by an independently certified layer whose probabilistic guarantee satisfies the assumption of the next. We sketch such an architecture and derive the compositional system-level safety bounds it admits via the chain rule of probability. Three open problems stand between this and a deployable
standard: bound estimation from non-i.i.d.\ traces, graceful degradation of contracts under deployment drift,  and
extension to multi-agent settings---the most important
unfinished business in LLM agent runtime assurance. 
\end{abstract}

\section{Introduction}
The deployment of LLM agents in safety-critical settings poses a
qualitatively new challenge for runtime assurance.
An LLM agent uses a language model as its reasoning engine to
implement task-oriented functionality: it receives instructions
from human users, conducts multi-step reasoning and planning, and
interacts with the environment through tool invocation or physical
execution.
Unlike conventional software, whose behaviour is formally
specified, and unlike static language models, whose outputs do
not act on the world, LLM agents combine non-deterministic
reasoning with closed-loop execution in which generated plans
influence the world and future observations.

\textbf{We argue a strong claim}: {safe deployment of LLM agents
is structurally inadequate under any single-layer enforcement
design, no matter how systematically engineered~\cite{dong2024guardrails}. The three dimensions of safe operation (semantic intent alignment and policy compliance, 
environmental validity, and dynamical feasibility) become
verifiable only at 
distinct stages of execution,
so any architecture that collapses them into a single
abstraction necessarily leaves at least one dimension
uncertified.
What is required is not a stronger guardrail, but a staged,
contract-based architecture in which three independently certified
layers (user, operational, functional) compose sequentially
through assume–guarantee reasoning}.
The limitation is 
not empirical but structural:
no amount of prompt engineering, model fine-tuning, or
execution monitoring confined to a single abstraction
can certify properties whose truth depends on information
becoming available only later in the execution chain.


LLMs are vulnerable to hallucinations, prompt injection,
distributional instability, and adversarial
manipulation~\cite{LLMtrustworthySurvey}.
When embedded in agents, these vulnerabilities are
amplified~\cite{li2025mind}: the consequences of an incorrect or
unsafe generation are no longer confined to textual output but
may propagate to physical actions, API calls, or system state
transitions.
\emph{The problem of safety in LLM agents is therefore not merely one
of content filtering but of runtime assurance in dynamic and
uncertain environments}.
Because the space of possible LLM agent failures is open-ended and cannot
be enumerated at design time, empirical evaluation alone cannot establish
the systematic assurance that safety-critical deployment demands.
Throughout this paper, \emph{guarantee} refers to probabilistic
assume--guarantee contracts---bounds on satisfaction probability rather
than deterministic proofs---a distinction that matters for interpreting
the framework's contributions against the broader literature on formal
verification of agents~\cite{miculicich2025veriguardenhancingllmagent,kamath2025enforcingtemporalconstraintsllm}.

That this gap has practical consequences is visible in the
benchmark record. Of sixteen popular agents evaluated on
AgentSafetyBench, none achieves a safety score above 60\%~\cite{zhang2024agentsafetybench},
with behavioural safety (30.4\%) significantly lagging behind
content safety (68.4\%). While these failures cannot be
attributed to a single cause—such as weak guardrails or
distributional mismatch—their pattern is instructive:
failures concentrate in environmental and physical dimensions,
beyond the reach of purely semantic controls.
A similar limitation appears from a security perspective.
Single-layer defences leave critical gaps, as reflected in
attack success rates exceeding 84\% on the Agent Security Bench.
Taken together, these observations do not constitute proof
of the structural argument developed below, but they are
consistent with it: both safety and security failures emerge
precisely where semantic-layer protections cease to apply.
This convergence motivates the need for a multi-layer
architecture, as developed in Section~\ref{sec:background-main}.

%

Recent work has begun to cross layer boundaries: VeriGuard~\cite{miculicich2025veriguardenhancingllmagent} 
achieves near-zero attack success rates by combining offline policy synthesis and formal verification  with lightweight  runtime monitoring, and
Agent-C~\cite{kamath2025enforcingtemporalconstraintsllm} reaches perfect temporal-constraint conformance
via SMT-enforced generation.
Yet both remain scoped to semantic intent and do not address environmental
validity or dynamical feasibility. Further complementary benchmarks, AgentHarm \cite{andriushchenko2024agentharm} and BAD-ACTS \cite{nother2025badacts}, have confirmed that agentic frameworks dramatically amplify misuse risk and that only multi-step monitoring protocols begin to render deployment plausible—consistent with the layered structure advocated here. 
While security and its limitations are briefly mentioned, this paper primarily focus on safety aspects.

The structure of the paper is as follows. Section \ref{sec:background-main} develops the structural argument formally. Section \ref{sec:framework} presents the three-layer instantiation and its safety bounds. Section \ref{sec:limitations} identifies four open problems that stand between that architecture and a deployable standard. It discusses alternative views in Section~\ref{sec:alternatives} and concludes in Section~\ref{sec:conclusions}.

\section{From Safety Dimensions to Contract-Based Design}\label{sec:background-main}



This section surveys existing approaches to LLM agent safety, identifies the 
structural characteristic that motivates our
framework, and argues for the three-layer architecture as the
minimal sufficient response.
 We focus on safety mechanisms operating \emph{externally} to the LLM weights, 
excluding fine-tuning approaches such as RLHF~\cite{NIPS2017_d5e2c0ad}, 
SFT~\cite{ouyang2022training}, and DPO~\cite{10.5555/3666122.3668460}, which require 
parameter modification and are often impractical when model weights are inaccessible. 
Instead, we focus on external guardrails that enforce safe behaviour at 
runtime, independently of the underlying model.

Unlike static text generation, LLM agents perform multi-step reasoning and planning, 
invoke tools, and interact with dynamic environments, causing errors to propagate 
beyond linguistic outputs to external systems, physical processes, or long-term state 
trajectories~\cite{li2025mind}.

\noindent\textbf{Existing Approaches and Their Limits}\label{sec:related}
%
%
The largest body of work targets \emph{semantic intent alignment and policy
compliance}. 
System-level approaches enforce constraints externally.
AgentSpec~\cite{wang2025agentspeccustomizableruntimeenforcement} introduces a lightweight
domain-specific language for runtime constraint enforcement, preventing
unsafe executions in over 90\% of code-agent cases while remaining
computationally lightweight.
Agent-C~\cite{kamath2025enforcingtemporalconstraintsllm} encodes temporal safety requirements
as first-order logic constraints enforced via SMT solving, achieving
perfect conformance on retail and airline benchmarks.
ShieldAgent~\cite{chen2025shieldagent} structures verifiable rules
extracted from policy documents into probabilistic rule circuits,
achieving 90.1\% rule recall.
These are genuine advances: they demonstrate that specification-driven
runtime guardrails are both practical and effective at the user
layer. The most systematic proposal within this single-layer paradigm is that 
of Dong et al.~\cite{dong2024guardrails}, who advocate a rigorous 
requirement-driven design process---covering neural-symbolic 
implementation, statistical certification, and a full systems 
development lifecycle---applied to a guardrail conceived as an 
input-output filter on the LLM. We do not dispute the value of this 
engineering rigour; we argue that no single-layer design, however 
carefully constructed, can certify all three safety dimensions, because 
the information required to do so becomes available only at strictly 
distinct stages of execution.

A second strand addresses the \emph{execution environment}.
Out-of-distribution detection~\cite{ren2023outofdistribution,zhang2025your}
identifies when model inputs deviate from training distribution, though
without determining whether the current world lies within the Operational Design
Domain (ODD)~\cite{iso34503,SAE_J3016_2021,BSI_PAS1883_2020}
under which system-level guarantees were derived.
Inner Monologue~\cite{huang2023inner} feeds environmental feedback back
into the LLM to support replanning, but reactively, once execution has
begun.
RoboGuard~\cite{ravichandran2026safety} and
Safety~Chip~\cite{yang2024plug} use world-model information to
instantiate action-level constraints via temporal logic---instrumentally,
rather than as an independent determination of whether execution is
authorised at all given the current world state, i.e.,\ whether the system is operating within its ODD.

A third strand targets \emph{low-level execution safety}.
Specification-based runtime monitoring~\cite{DBLP:series/lncs/BartocciDDFMNS18} tracks
system trajectories against temporal logic specifications.
Control barrier functions~\cite{DBLP:conf/eucc/AmesCENST19} enforce forward invariance
of a safe set with minimal intervention, though their guarantees degrade
when environmental assumptions break at runtime.
Shielding in reinforcement
learning~\cite{DBLP:conf/aaai/AlshiekhBEKNT18,DBLP:conf/ijcai/YangMRR23}
synthesises correctors from
temporal logic specifications; both absolute and probabilistic variants
require knowledge of MDP safety dynamics---an assumption LLM agent
environments structurally violate.
Recent work integrating CBFs with LLM
planners~\cite{khan2025safer,wang2025robosafe,zhao2026sagellmsafegeneralizablellm}
demonstrates that functional-layer enforcement is independently
necessary and cannot be subsumed by user-level guardrails.

Recent efforts attempt broader coverage.
Agent Behavioral Contracts~\cite{bhardwaj2026abc} introduce
probabilistic $(p,\delta,k)$-satisfaction notions---meaning a contract holds with probability at least $p$, confidence
$1{-}\delta$, over windows of $k$ steps---with a Drift Bounds
Theorem, detecting 5.2--6.8 soft violations per session that
uncontracted baselines miss entirely.
\mbox{PRO\textsuperscript{2}GUARD}~\cite{wang2025proguard} extends
enforcement into the probabilistic domain via learned Discrete-Time
Markov Chains, achieving PAC-correctness on risk estimates.
Shamsujjoha et al.~\cite{shamsujjoha2024swiss} contribute a systematic taxonomy of runtime
guardrails for Foundation Model based agents, grounded in a systematic literature review and structured
around three dimensions: quality attributes, pipeline stages, and artefacts. This is the
closest antecedent to the present work and clearly motivates multi-layered thinking. The
key difference is organisational principle: their taxonomy is \emph{artefact-driven},
decomposing guardrail design by which pipeline objects (goals, plans, tools, outputs) are
being guarded at each stage. Our architecture is \emph{information-driven}, deriving
layer boundaries from a structural argument about \emph{when} each category of safety
claim can first be verified---given user instructions, policies and rules, before any world observation (User Layer)
, given current
sensor data (Operational Layer), or continuously during actuation (Functional Layer).
This distinction is not merely taxonomic: it grounds the contract-chain structure and the
compositional probability bounds of Section~\ref{sec:framework}, which depend on the
sequential availability of information rather than on the artefact being inspected. 

An artefact-driven decomposition cannot ensure that certifications relying on temporally distinct information sets remain consistent. 
This breaks the assumption–guarantee chain, as guarantees are no longer established against the correct information domain. 
The information-driven decomposition avoids this issue by construction and enables compositional safety bounds~\ref{sec:bounds}.

\noindent\textbf{Contract-Based Design as the Unifying Principle}
\label{sec:contracts}
The structural characteristic identified above 
 demands a response with three properties simultaneously,
 each corresponding to a structural failure mode of
 single-layer enforcement: it must reason across
 heterogeneous verification methods without forcing them
 into a common formalism; it must compose guarantees across layers whose
dynamics are non-deterministic and non-stationary; and it must yield
modular statistical bounds certifiable layer by layer.
We argue that probabilistic assume--guarantee (A/G)
contracts~\cite{delahaye2011probabilistic,DBLP:conf/acsd/DelahayeCL10,%
hampus2024probabilistic, blohm2024towards} provide exactly this foundation. 

Heterogeneous verification methods are accommodated by the contract
structure itself.
Each safety component is specified by an assumption on its inputs and a
guarantee on its outputs, scoped to its own abstraction level and
information set.
Semantic, operational, and dynamical dimensions need not be unified
into a single model~\cite{giannakopoulou2018compositional}: each layer
carries an independently verifiable contract, and the heterogeneity of
verification methods and specification formalisms across layers is a feature, not a limitation.
This matters in practice: existing CPS assumptions are
typically informal and split between deterministic and
probabilistic formalisations~\cite{li2025systematizing}, a fragmentation
that worsens when one layer is an LLM---a gap we begin to
address here.

Principled composition across non-stationary layers follows from the
quantitative semantics of contracts.
Satisfaction is probabilistic: each contract holds with some
probability, meaning that executions satisfying the assumption also
satisfy the guarantee with that probability. 
Sequential composition is then well-defined: the chain rule
decomposes system-level safety into a product of conditional
probabilities aligned with the architectural
stages~\cite{delahaye2011probabilistic}, and the
Fr\'{e}chet--Bonferroni bound yields a non-trivial
lower bound from marginals alone whenever individual probabilities are
sufficiently high.
Critically, this rule does not presuppose deterministic or stationary
subsystem dynamics, a necessary distinction from existing CPS contract
frameworks---which assume well-characterised noise models---rather than
the discrete token sampling and context-window non-determinism
characteristic of LLMs. Modular statistical guarantees follow from the same decomposition: 
A/G contracts compose locally certifiable bounds into a system-level 
guarantee without global analysis, complementing PAC-style statistical 
interpretability as a refinement adapted to the non-stationary, 
finite-trace setting of LLM agents; the estimation challenges this 
raises are developed in Section~\ref{sec:bounds}.

Within each layer, neural-symbolic methods should be adopted
\cite{chen2025shieldagent,kamath2025enforcingtemporalconstraintsllm}: the symbolic
component renders contract assumptions and guarantees formally
checkable, while the neural component handles the perceptual and
linguistic complexity that purely symbolic approaches cannot scale to.

\noindent\textbf{The Three-Layer Architecture Is the Necessary Response}
\label{sec:position}
Let $\mathcal{I}_U$, $\mathcal{I}_O$, and $\mathcal{I}_F$ denote the three information sets 
available respectively (i) prior to world observation,
(ii) after world-state estimation but before actuation,
and (iii) during control-loop execution.
$\mathcal{I}_U$ contains user intent, policy, and role metadata but
no sensor-derived state.
$\mathcal{I}_O$ contains the estimated world state $\hat{w}$ and ODD
configuration but lacks real-time dynamical trajectory information.
$\mathcal{I}_F$ contains full state trajectories and control inputs but
operates after both intent validation and ODD authorisation.
These sets are strictly ordered,
$\mathcal{I}_U \not\subseteq \mathcal{I}_O \not\subseteq \mathcal{I}_F$
(Appendix~\ref{app:impossibility}), and mutually
non-substitutable as certification domains. 
%

Let $\Phi_U(e)$, $\Phi_O(e)$, and $\Phi_F(e)$ denote the Boolean safety predicates for semantic, operational, and dynamical safety respectively (formally defined in Appendix~\ref{app:impossibility}).
Because $\Phi_U$, $\Phi_O$, and $\Phi_F$ depend on strictly
distinct information sets that become available at different
stages, any architecture that collapses
them into fewer than three independently certified stages must
either violate the desired semantic, operational or dynamical safety (see the desiderata D1--D3 of Appendix~\ref{app:impossibility}) or
implicitly reconstruct the three-stage structure within a
single component (Proposition~\ref{prop:collapse}).
The minimality of three layers is therefore not a design choice
but a consequence of the temporal ordering $\tau_U < \tau_O <
\tau_F$ of certifiable information, where $\tau_i$ denotes the
earliest time at which $\Phi_i$ can be certified: fewer layers break the
contract chain or hide its boundaries; additional layers may
refine it but cannot reduce this minimum
(Appendix~\ref{app:impossibility}, Corollary~\ref{cor:three-necessary}).

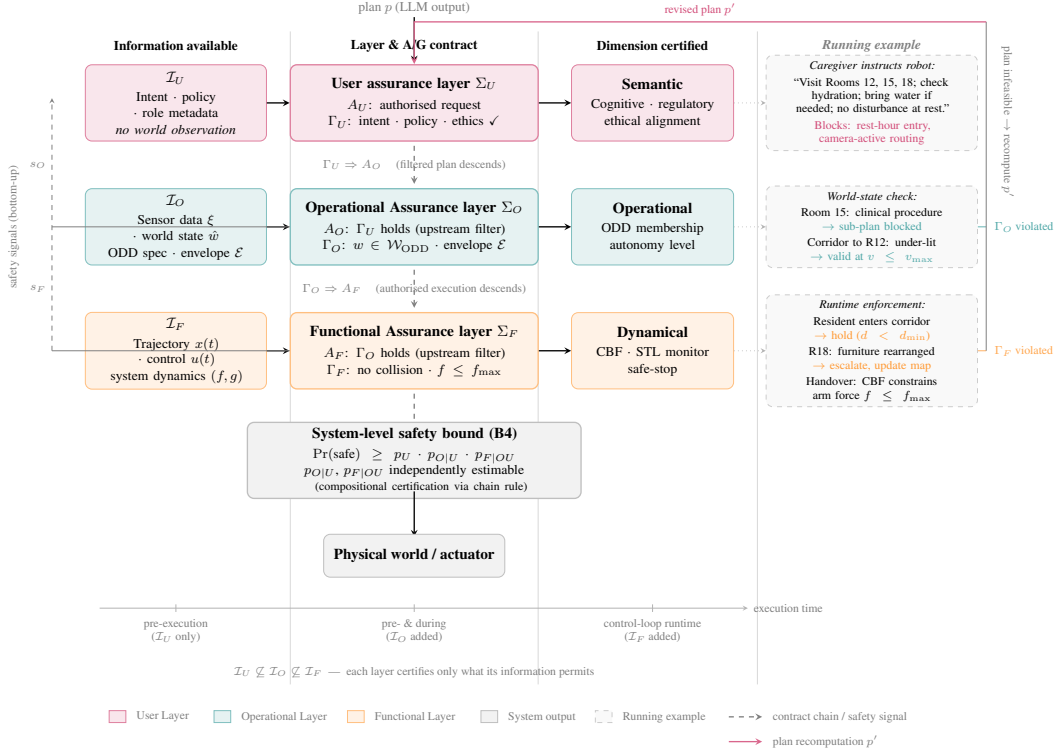
\begin{figure}[http]
\centering
\resizebox{\linewidth}{!}{%
\begin{tikzpicture}[
  font=\small,
  >=stealth,
  every node/.style={align=center},
  layerbox/.style={
    draw, rounded corners=4pt, thick,
    minimum width=5.2cm, minimum height=1.6cm,
    text width=4.8cm
  },
  infobox/.style={
    draw, rounded corners=4pt,
    minimum width=3.8cm, minimum height=1.6cm,
    text width=3.4cm
  },
  dimbox/.style={
    draw, rounded corners=4pt,
    minimum width=3.4cm, minimum height=1.6cm,
    text width=3.0cm
  },
  exbox/.style={
    draw=gray!40, dashed, rounded corners=4pt,
    fill=gray!5,
    minimum width=4.3cm, minimum height=1.6cm,
    text width=4.2cm,
    font=\scriptsize
  },
  ulayer/.style={fill=purple!10, draw=purple!50},
  olayer/.style={fill=teal!10,   draw=teal!50},
  flayer/.style={fill=orange!10, draw=orange!50},
  sysbox/.style={fill=gray!10,   draw=gray!50},
  contractarrow/.style={dashed, ->, thick, gray},
  feedbackarrow/.style={dashed, ->, gray},
  solidarrow/.style={->, thick},
  leaderarrow/.style={dotted, gray!60, thin, ->}
]

\def\xinfo{0}
\def\xlayer{5.0}
\def\xdim{10.0}
\def\xex{14.6}     
\def\yU{0}
\def\yO{-2.6}
\def\yF{-5.2}
\def\ybound{-7.5}
\def\yworld{-9.5}

\node[font=\footnotesize\bfseries] at (\xinfo,  1.2) {Information available};
\node[font=\footnotesize\bfseries] at (\xlayer, 1.2) {Layer \& A/G contract};
\node[font=\footnotesize\bfseries] at (\xdim,   1.2) {Dimension certified};
\node[font=\footnotesize\bfseries\itshape, gray] at (\xex, 1.2) {Running example};

\draw[gray!30, thin] (2.4,  1.5) -- (2.4,  -12.2);
\draw[gray!30, thin] (7.6,  1.5) -- (7.6,  -12.2);
\draw[gray!30, thin] (12.2, 1.5) -- (12.2, -12.2);

\node[font=\footnotesize, gray] at (\xlayer, 2.0)
  {plan $p$ (LLM output)};
\draw[solidarrow] (\xlayer, 1.8) -- (\xlayer, 0.85);

\node[infobox, ulayer] (IU) at (\xinfo, \yU) {%
  $\mathcal{I}_U$\\[2pt]
  \footnotesize Intent $\cdot$ policy $\cdot$ role metadata\\
  \textit{no world observation}};

\node[layerbox, ulayer] (LU) at (\xlayer, \yU) {%
  \textbf{User assurance layer} $\Sigma_U$\\[3pt]
  \footnotesize $A_U$: authorised request\\
  $\Gamma_U$: intent $\cdot$ policy $\cdot$ ethics \checkmark};

\node[dimbox, ulayer] (DU) at (\xdim, \yU) {%
  \textbf{Semantic}\\[2pt]
  \footnotesize Cognitive $\cdot$ regulatory\\ethical alignment};

\node[exbox] (EU) at (\xex, \yU) {%
  \textit{Caregiver instructs robot:}\\[2pt]
  ``Visit Rooms 12, 15, 18;
  check hydration; bring
  water if needed; no
  disturbance at rest.''\\[2pt]
  \textcolor{purple!70}{Blocks: rest-hour entry,}\\
  \textcolor{purple!70}{camera-active routing}};

\draw[solidarrow]   (IU.east) -- (LU.west);
\draw[solidarrow]   (LU.east) -- (DU.west);
\draw[leaderarrow]  (DU.east) -- (EU.west);

\node[font=\scriptsize, gray] at (\xlayer, -1.3)
  {$\Gamma_U \Rightarrow A_O$ \quad (filtered plan descends)};
\draw[contractarrow] (\xlayer, -0.9) -- (\xlayer, \yO+0.9);

\node[infobox, olayer] (IO) at (\xinfo, \yO) {%
  $\mathcal{I}_O$\\[2pt]
  \footnotesize Sensor data $\xi$ $\cdot$ world state $\hat{w}$\\
  ODD spec $\cdot$ envelope $\mathcal{E}$};

\node[layerbox, olayer] (LO) at (\xlayer, \yO) {%
  \textbf{Operational Assurance layer} $\Sigma_O$\\[3pt]
  \footnotesize $A_O$: $\Gamma_U$ holds (upstream filter)\\
  $\Gamma_O$: $w \in \mathcal{W}_\mathrm{ODD}$ $\cdot$ envelope $\mathcal{E}$};

\node[dimbox, olayer] (DO) at (\xdim, \yO) {%
  \textbf{Operational}\\[2pt]
  \footnotesize ODD membership\\autonomy level};

\node[exbox] (EO) at (\xex, \yO) {%
  \textit{World-state check:}\\[2pt]
  Room 15: clinical procedure\\
  \textcolor{teal!70}{$\to$ sub-plan blocked}\\[2pt]
  Corridor to R12: under-lit\\
  \textcolor{teal!70}{$\to$ valid at $v \leq v_{\max}$}};

\draw[solidarrow]   (IO.east) -- (LO.west);
\draw[solidarrow]   (LO.east) -- (DO.west);
\draw[leaderarrow]  (DO.east) -- (EO.west);

\node[font=\scriptsize, gray] at (\xlayer, \yO-1.3)
  {$\Gamma_O \Rightarrow A_F$ \quad (authorised execution descends)};
\draw[contractarrow] (\xlayer, \yO-0.9) -- (\xlayer, \yF+0.9);

\node[infobox, flayer] (IF) at (\xinfo, \yF) {%
  $\mathcal{I}_F$\\[2pt]
  \footnotesize Trajectory $x(t)$ $\cdot$ control $u(t)$\\
  system dynamics $(f,g)$};

\node[layerbox, flayer] (LF) at (\xlayer, \yF) {%
  \textbf{Functional Assurance layer} $\Sigma_F$\\[3pt]
  \footnotesize $A_F$: $\Gamma_O$ holds (upstream filter)\\
  $\Gamma_F$: no collision $\cdot$ $f \leq f_\mathrm{max}$};

\node[dimbox, flayer] (DF) at (\xdim, \yF) {%
  \textbf{Dynamical}\\[2pt]
  \footnotesize CBF $\cdot$ STL monitor\\safe-stop};

\node[exbox] (EF) at (\xex, \yF) {%
  \textit{Runtime enforcement:}\\[2pt]
  Resident enters corridor\\
  \textcolor{orange!80}{$\to$ hold ($d < d_{\min}$)}\\[2pt]
  R18: furniture rearranged\\
  \textcolor{orange!80}{$\to$ escalate, update map}\\[2pt]
  Handover: CBF constrains\\
  arm force $f \leq f_{\max}$};

\draw[solidarrow]   (IF.east) -- (LF.west);
\draw[solidarrow]   (LF.east) -- (DF.west);
\draw[leaderarrow]  (DF.east) -- (EF.west);

\draw[contractarrow] (\xlayer, \yF-0.9) -- (\xlayer, \ybound+0.55);

\node[sysbox, draw, rounded corners=4pt, thick,
      minimum width=7.0cm, minimum height=1.2cm,
      text width=6.6cm] (BOUND) at (\xlayer, \ybound) {%
  \textbf{System-level safety bound (B4)}\\[2pt]
  \footnotesize
  $\Pr(\text{safe}) \geq p_U \cdot p_{O|U} \cdot p_{F|OU}$\\
  $p_{O|U}$, $p_{F|OU}$ independently estimable\\
  \quad{\scriptsize(compositional certification via chain rule)}};

\draw[solidarrow] (\xlayer, \ybound-0.65) -- (\xlayer, \yworld+0.4);

\node[sysbox, draw, rounded corners=4pt,
      minimum width=3.8cm, minimum height=0.9cm] (WORLD)
  at (\xlayer, \yworld) {\textbf{Physical world / actuator}};

\def\xfb{-2.6}

\draw[gray, thin] (LF.west) -- ++(-1.05, 0) -- (\xfb, \yF);
\draw[gray, thin] (LO.west) -- ++(-1.05, 0) -- (\xfb, \yO);
\draw[feedbackarrow] (\xfb, \yF) -- (\xfb, \yU+0.1);

\node[font=\scriptsize, gray, left] at (\xfb, {(\yF+\yO)/2}) {$s_F$};
\node[font=\scriptsize, gray, left] at (\xfb, {(\yO+\yU)/2}) {$s_O$};
\node[font=\scriptsize, gray, rotate=90]
  at (\xfb-0.8, {(\yF+\yU)/2}) {safety signals (bottom-up)};

\def\xrb{17.0}   
\def\yover{1.7}  

\draw[orange!60, thin] (EF.east) -- (\xrb, \yF);
\draw[teal!60,   thin] (EO.east) -- (\xrb, \yO);
\draw[gray, thin]      (\xrb, \yF) -- (\xrb, \yover);
\draw[->, thick, purple!60]
  (\xrb, \yover) -- (\xlayer, \yover) -- (LU.north);

\node[font=\scriptsize, gray, rotate=-90]
  at (\xrb+0.45, {(\yF+\yU)/2+2.2}) {plan infeasible $\to$ recompute $p'$};

\node[font=\scriptsize, purple!70, above] at ({(\xrb+\xlayer)/2}, \yover)
  {revised plan $p'$};

\node[font=\scriptsize, orange!70, right] at (\xrb+0.05, \yF)
  {$\Gamma_F$ violated};
\node[font=\scriptsize, teal!70,   right] at (\xrb+0.05, \yO)
  {$\Gamma_O$ violated};

\def\ytl{-10.6}

\draw[->, gray!60, thin] (-1.6, \ytl) -- (12.0, \ytl)
  node[right, font=\scriptsize, gray]{execution time};

\draw[gray!60, thin] (\xinfo,  \ytl+0.08) -- (\xinfo,  \ytl-0.08);
\node[font=\scriptsize, gray, below] at (\xinfo, \ytl-0.1)
  {pre-execution\\($\mathcal{I}_U$ only)};

\draw[gray!60, thin] (\xlayer, \ytl+0.08) -- (\xlayer, \ytl-0.08);
\node[font=\scriptsize, gray, below] at (\xlayer, \ytl-0.1)
  {pre- \& during\\($\mathcal{I}_O$ added)};

\draw[gray!60, thin] (\xdim,   \ytl+0.08) -- (\xdim,   \ytl-0.08);
\node[font=\scriptsize, gray, below] at (\xdim, \ytl-0.1)
  {control-loop runtime\\($\mathcal{I}_F$ added)};

\node[font=\scriptsize, gray, below] at (\xlayer, \ytl-1.1)
  {
  $\mathcal{I}_U \not\subseteq \mathcal{I}_O \not\subseteq \mathcal{I}_F$
   \;---\; each layer certifies only what its information permits};

\def\yleg{-13}
\def\xleg{-1.4}
\def\swid{0.35}
\def\shei{0.25}

\filldraw[fill=purple!10, draw=purple!50]
  (\xleg, \yleg) rectangle +(\swid, \shei);
\node[font=\scriptsize, gray, right]
  at (\xleg+\swid+0.1, \yleg+0.12) {User Layer};

\filldraw[fill=teal!10, draw=teal!50]
  (\xleg+2.2, \yleg) rectangle +(\swid, \shei);
\node[font=\scriptsize, gray, right]
  at (\xleg+2.2+\swid+0.1, \yleg+0.12) {Operational Layer};

\filldraw[fill=orange!10, draw=orange!50]
  (\xleg+5.0, \yleg) rectangle +(\swid, \shei);
\node[font=\scriptsize, gray, right]
  at (\xleg+5.0+\swid+0.1, \yleg+0.12) {Functional Layer};

\filldraw[fill=gray!10, draw=gray!50]
  (\xleg+7.8, \yleg) rectangle +(\swid, \shei);
\node[font=\scriptsize, gray, right]
  at (\xleg+7.8+\swid+0.1, \yleg+0.12) {System output};

\filldraw[fill=gray!5, draw=gray!40, dashed]
  (\xleg+10.2, \yleg) rectangle +(\swid, \shei);
\node[font=\scriptsize, gray, right]
  at (\xleg+10.2+\swid+0.1, \yleg+0.12) {Running example};

\draw[contractarrow]
  (\xleg+13.0, \yleg+0.12) -- ++(0.7, 0);
\node[font=\scriptsize, gray, right]
  at (\xleg+13.8, \yleg+0.12) {contract chain / safety signal};

\draw[->, thick, purple!60]
  (\xleg+13.0, \yleg-0.4) -- ++(0.7, 0);
\node[font=\scriptsize, gray, right]
  at (\xleg+13.8, \yleg-0.4) {plan recomputation $p'$};

\end{tikzpicture}%
}
\caption{Three-layer probabilistic assume--guarantee architecture with
running example (right column).
Each layer $\Sigma_i$ is certified against the
information set $\mathcal{I}_i$ first available at its execution stage.
Forward contract propagation is shown by dashed downward arrows
($\Gamma_U\Rightarrow A_O$, $\Gamma_O\Rightarrow A_F$).
Left margin: upward safety signals ($s_F\to s_O\to s_U$).
Right margin: when $\Gamma_O$ or $\Gamma_F$ is violated the User Layer
recomputes a revised plan $p'$, making the architecture a live
bidirectional assurance loop.\vspace{-2em}}
\label{fig:three-layer}
\end{figure} 


\section{Three-Layer Framework}
\label{sec:framework}


The three-layer framework, illustrated in Figure~\ref{fig:three-layer}, instantiates the contract-based design of
Section~\ref{sec:contracts}. 
A plan~$p$ passes sequentially through three assurance barriers before and during execution.  Each barrier $i \in \{U, O, F\}$ carries a probabilistic assume--guarantee (A/G) contract
$\Sigma_i = (A_i, \Gamma_i)$: if the environment satisfies assumption~$A_i$, the layer guarantees property~$\Gamma_i$ with probability~$p_i = \Pr(\Gamma_i)$~\cite{delahaye2011probabilistic}.\martinComment{I am afraid that in AI practice, we can't really apply the probabilistic contract theory of Delahaye, Caillaud, and Legay  \cite{delahaye2011probabilistic}, unless we accept trivial/vacuous and thus perfectly useless bounds on contract satisfaction probabilities, namely a satisfaction probability of 0. The problem is that the input to, e.g., an DNN is open, i.e., a non-deterministic choice. Delahay et al.'s notion of contract satisfaction takes infimum over all schedulers (i.e., over the non-deterministic NN inputs in our case) of the probability of contract satisfaction \cite[Def.~7]{delahaye2011probabilistic}. With the DNN being epistemically, not aleatory, uncertain, this probability infimum will be 0 if there just is a single input that the NN misperforms (in the sense of providing wrong output) on. Ouch! 

N.B.: This does not mean that the theory of \cite{delahaye2011probabilistic} is flawed; the theory is perfectly sound. But it does not buy us anything in our context, as virtually all probability-of-satisfaction bounds we can confirm for our components will be 0.}
\martinComment{Note that the apparent alternative, namely to treat inputs as probabilistic rather than nondeterministic choices, in practice is none: This would fix the input distribution such that contract-based reasoning does not carry over to even the slightest distribution shift --- we'd totally loose the benefit of contracts, namely that their reasoning applies to a variety of possible contexts, namely to all environments satisfying the assumption. 

Blends of nondeterministic and probabilistic choice could in principle alleviate this problem, but are hard to argue and instantiate from an application perspective; they'd be a coding trick whose details, including concrete parameters of the probability distributions, do not follow immediately from the application.}
\martinComment{I'd suggest to resolve or at least alleviate the issue by the following sentence.}
These probabilistic assume-guarantee frameworks or their refinements addressing conditional distributions \cite{BlohmFHKR24} facilitate the relevant modular reasoning.
The guarantees chain, derived from the temporal ordering as $\Gamma_U \Rightarrow A_O$ and $\Gamma_O \Rightarrow A_F$, enables a  modular decomposition of the system-level safety probability via the chain rule introduced in Section~\ref{sec:bounds}.

\noindent\textbf{Running example.}  A caregiver instructs a service robot (LLM-orchestrated):
\textit{``Visit Rooms~12, 15, and~18 before 15:00; check hydration
and wellbeing; bring water if needed; do not disturb during rest
periods.''} The LLM produces plan~$p$---a sequence of visit,
assessment, and conditional-delivery actions---which passes through
the three layers as developed below (see Figure~\ref{fig:three-layer}).


\subsection{The Three Layers}
\label{sec:layers}

\noindent\textbf{User Assurance Layer~($\Sigma_U$).}
The first barrier validates~$p$ against intent, policy, and ethics
\emph{before any world observation is made}, across three orthogonal
dimensions that together span the space of pre-execution
failures~\cite{li2025mind,shamsujjoha2024swiss}:
(i)~\emph{cognitive alignment}---does the plan faithfully realise the
user's intent, neither under- nor over-delivering?~\cite{wang2025agentspeccustomizableruntimeenforcement};
(ii)~\emph{regulation alignment}---does it conform to domain-specific
rules and policies?~\cite{kamath2025enforcingtemporalconstraintsllm,wang2025agentspeccustomizableruntimeenforcement};
and (iii)~\emph{ethical alignment}---does it respect ethical
constraints such as prohibiting cameras in private
spaces?~\cite{Bai2022ConstitutionalAH,chen2025shieldagent}.
Authorisation of the issuing user is verified in parallel.  A plan
failing any check is blocked or, where recoverable, returned for
clarification.

Beyond gating, the layer derives quantitative constraints---bounds 
$\text{caps}_U$ on continuous variables such as speed, force, and proximity, 
and exclusions $\text{zones}_U$ over spatial regions, time windows, and 
task types---from user roles and ethical rules, passing them downstream.  Assuming $A_U$ (authorised request), the
layer guarantees $\Gamma_U$: the emitted plan targets only authorised
locations, encodes conditional actions correctly, and contains no
policy or ethical violations; $\Gamma_U$ constitutes ~$A_O$.

\textit{Example.}  The layer rejects plans omitting the conditional
water-delivery logic (cognitive), scheduling room entry during rest
hours without override (regulation), or routing through a bathroom
with cameras active (ethical).

Key open problems include automated formalisation of informal intent and
domain policies into verifiable representations---initial attempts in 
ShieldAgent~\cite{chen2025shieldagent} and
Agent-C~\cite{kamath2025enforcingtemporalconstraintsllm} but not yet closed.

\noindent\textbf{Operational Assurance Layer~($\Sigma_O$).}
The second barrier asks not ``what should the agent do?'' but 
“is this a world in which we are certified to act, and under what degree
of autonomy?” Let $\mathcal{W}$ denote the space of all possible world states
reachable by the agent during deployment.
The Operational Design Domain (ODD) is a certified subset
$\mathcal{W}_{\mathrm{ODD}} \subseteq \mathcal{W}$ encoding five independently certifiable
dimensions: physical configuration (geometry, obstacles, dynamics),
perceptual configuration (sensor availability, visibility),
contextual configuration (time of day, occupancy, access restrictions),
governance configuration (authorisation regimes, policy zones),
and social or normative context (restricted roles, privacy constraints).
The first three follow established ODD formalisms~\cite{iso34503}; the latter two
extend them to the institutional embedding of LLM agents.
Execution outside $\mathcal{W}_{\mathrm{ODD}}$ invalidates all downstream guarantees
regardless of physical capability.

Beyond ODD membership, the Operational Layer governs the agent’s
\emph{autonomy envelope}: the degree of freedom under which the plan
may be executed in the current operational context.
We model this envelope as
$\mathcal{E}(\hat{w}) = (L, \Pi, H)$,
where $L$ denotes the autonomy level, $\Pi$ the permitted action and
tool set, and $H$ the maximum execution horizon before checkpoint.
While ODD validity determines whether execution is authorised at all,
the autonomy envelope determines how much autonomy is granted
within the valid domain.

The layer ingests sensor data~$\xi$, computes world-state
estimate~$\hat{w}$, and issues a deterministic validity verdict.
Determinism preserves audit trails and compositional
invariants~\cite{delahaye2011probabilistic}; probabilistic uncertainty
is handled upstream, but the verdict itself is binary. Note that, while the  verdict is binary for every input, a probabilistic guarantee is needed when dealing with a population of inputs.  When the ODD
is invalid or degraded, a fallback policy triggers plan restriction or
human takeover.  
Assuming $A_O$ (i.e., $\Gamma_U$ holds), the layer guarantees $\Gamma_O$:
(i) execution remains within $\mathcal{W}_{\text{ODD}}$,
(ii) all ODD-invalid sub-plans are blocked, and
(iii) autonomy-envelope updates are conservative with respect to
operational degradation (i.e., autonomy reduction, restriction of
permitted actions, or activation of a Minimal Risk Condition (MRC)~\cite{SAE_J3016_2021} when required).
This guarantee constitutes $A_F$ for the Functional Layer.

\textit{Example.}  Room~15 is restricted (clinical procedure): that
sub-plan is blocked.  The corridor to Room~12 is under-lit but within
the ODD subject to~$v \leq v_{\max}$.

Key open problems include maintaining guarantees under distributional
shift: when $\hat{w}$ drifts from the distribution under which
$\mathcal{W}_{\mathrm{ODD}}$ was certified,  methods such as~\cite{zhao2024robust} may be needed.

\noindent\textbf{Functional Assurance Layer~($\Sigma_F$).}
The third barrier enforces \emph{how} the system executes at
control-loop frequency.  Three complementary mechanisms share the
load.  \emph{Specification-based runtime
monitoring}~\cite{DBLP:series/lncs/BartocciDDFMNS18} tracks the trajectory
against safety specifications, producing robustness margins
$\rho_F(t)$---the signal temporal logic (STL)
robustness degree~\cite{DBLP:series/lncs/BartocciDDFMNS18} of the current trajectory against the
safety specification---and triggering intervention when a violation is detected
or anticipated.  \emph{Control barrier functions
(CBFs)}~\cite{DBLP:conf/eucc/AmesCENST19} project any candidate control
command~$u_{\mathrm{des}}(t)$ onto the admissible safe set, correcting
it minimally; recent work confirms CBF quadratic programs embed inside
LLM planning loops without sacrificing task
completion~\cite{khan2025safer,zhao2026sagellmsafegeneralizablellm}, and probabilistic
extensions now bound safety failure probability for learned
barriers~\cite{urrea2026probabilistic,mestres2025probabilistic}.
\emph{Simulation-based
synthesis}~\cite{DBLP:conf/aaai/AlshiekhBEKNT18,DBLP:conf/ijcai/YangMRR23} pre-computes safe
envelopes via world models when the environment is large or unknown;
CBF projection takes precedence when the two conflict, as it provides
a certified instantaneous correction independent of model accuracy.

Assuming $A_F$ ($\Gamma_O$ holds: $w \in \mathcal{W}_{\mathrm{ODD}}$,
$v \leq v_{\max}$), the layer guarantees $\Gamma_F$: no collision
occurs, all physical interactions satisfy $f \leq f_{\max}$, and any
violation triggers corrective action or safe-stop.  Safety signals $s_i$ are structured status reports emitted by each layer upon 
constraint violation or execution failure, indicating whether operation is 
nominal or requires upstream intervention. They propagate bottom-up: the functional safety signal~$s_F$ is reported to
the Operational Layer, which propagates~$s_O$ to the User
Layer~\cite{giannakopoulou2018compositional}, enabling plan
recomputation when execution cannot complete safely.  The contract
chain is thus a live, \emph{bidirectional assurance loop}.

\textit{Example.}  A resident entering the corridor triggers a hold; a
CBF constrains arm force during the handover; rearranged furniture in
Room~18 causes escalation with an updated obstacle map, exercising the
bottom-up channel.

Key open problems include real-time safety under non-stationary
dynamics~\cite{DBLP:series/lncs/BartocciDDFMNS18} and mid-execution adaptation
of safety envelopes when updated caps or zones arrive from upper
layers.
\subsection{End-to-End Safety Guarantee}
\label{sec:bounds}
\vspace{-5pt}
Since $\Gamma_U \Rightarrow A_O$ and $\Gamma_O \Rightarrow A_F$, the
three contracts compose sequentially~\cite{delahaye2011probabilistic,giannakopoulou2018compositional}. 
Let $F_i = \neg\Gamma_i$ be the failure event of layer~$i$.  
The system-level safety probability $\Pr(\text{system safe}) = \Pr(\Gamma_U \cap \Gamma_O \cap \Gamma_F)$: admits four characterisations of increasing precision:

\begin{align}
  \Pr(\text{safe}) &\geq \max(0,\; p_U + p_O + p_F - 2)
    \tag{B1}\label{eq:B1}\\
  \Pr(\text{safe}) &\geq \max(0,\; p_U + p_O + p_F - 2
    + \textstyle\sum_{i<j}\Pr(F_i\cap F_j))
    \tag{B2}\label{eq:B2}\\
  \Pr(\text{safe}) &= 1 - \Pr(F_U\cup F_O\cup F_F)
    \tag{B3}\label{eq:B3}\\
  \Pr(\text{safe}) &= p_U \cdot p_{O|U} \cdot p_{F|OU}
    \tag{B4}\label{eq:B4}
\end{align}

(B1) provides a conservative lower bound from marginals alone, while (B2) incorporates pairwise co-failure probabilities. (B3) is exact via inclusion–exclusion. (B4), derived from the chain rule, is the most informative form, as its factors align with the architectural stages and correspond to layer-wise conditional guarantees.

These bounds do not assume independence. Instead, they reflect different levels of available statistical information, from marginals to conditional probabilities. The architecture’s key contribution is to make the factors of (B4) well-defined and independently estimable by aligning them with the sequential availability of information across layers.

Any weakening of any single layer degrades all four bounds; no strengthening of one compensates for a gap in another~\cite{delahaye2011probabilistic}. Appendix~\ref{app:numerical} instantiates these bounds on the running example with illustrative estimates; Proposition~\ref{prop:bound-degradation} in Appendix~\ref{sec:bound-degradation} establishes that no genuine two-stage design yields \eqref{eq:B4} as a product of independently estimable, auditable layer-level quantities. The bounds are further elaborated in Appendix~A.

%

The running example instantiates the contract chain $\Gamma_U$ delivers $A_O$, 
$\Gamma_O$ delivers $A_F$, and $\Gamma_F$ .The Room~18 escalation illustrates 
the architecture's bidirectional character: the Operational Layer reissues 
revised constraints on a plan already certified by $\Sigma_U$, and the 
Functional Layer operates on an execution context already authorised by both 
upper layers. Whether these upstream certifications translate into quantitative 
improvements ($p_{O|U} > p_O$, $p_{F|OU} > p_F$) is deployment-dependent; 
what the architecture guarantees is that these conditionals are well-defined, 
independently estimable, and auditable---properties that any two-stage collapse 
would forfeit (Proposition~\ref{prop:bound-degradation}).

\section{Limitations and Future Directions}\label{sec:limitations}






\noindent\textbf{From A/G Guarantee to Estimable Bounds}\label{sec:pac-guarantees}
The framework treats the LLM as a fixed black box, which has a direct consequence for certification: if the underlying model is updated or replaced, the contract structure carries over but the layer-level probabilities $p_U, p_O, p_F$ and their conditionals must be re-derived from fresh execution traces, since contract satisfaction rates will generally change. More fundamentally, even for a fixed model, the bounds \eqref{eq:B1}--\eqref{eq:B4} require estimating the marginal probabilities $p_U, p_O, p_F$; the pairwise co-failure rates $\Pr(F_i \cap F_j)$; and the conditionals $p_{O|U}, p_{F|OU}$. None of these is analytically available for LLM agents, because token sampling introduces
non-stationarity---identical instructions may yield different plans across invocations---making
standard i.i.d.\ assumptions untenable for bound estimation. The numerical instantiation of Appendix~\ref{app:numerical}
  illustrates what is at stake: a 4-point estimation error in
   $p_{O|U}$ propagates directly to the (B4) certificate,
   motivating tight rather than merely valid bounds. Standard PAC theory~\cite{valiant1984} is therefore not directly applicable, for three compounding reasons.

The first is \emph{non-i.i.d.\ traces}. Each step in an LLM agent trace conditions on prior context, violating the independence assumption that PAC bounds require. For the marginal estimates required by~\eqref{eq:B1}, martingale bounds~\cite{lotfi2024tokens} and mixing-process bounds~\cite{mohri2009,
mohri2010} offer the most applicable partial remedies, tolerating within-sequence dependence without stationarity assumptions. For the conditional quantities required by~\eqref{eq:B4}, non-exchangeable conformal prediction~\cite{barber2023} and the survival-analysis calibration \cite{davidov2026}---which reframes safety-probability estimation as a time-to-unsafe-sampling problem and constructs distribution-free lower predictive bounds---are better suited. Anytime-valid inference via $e$-processes~\cite{ramdas2023} applies across all quantities but yields wider intervals. None of these frameworks has yet been extended to variable-length, layer-level satisfaction events across heterogeneous layers, and ScenicProver~\cite{vin2025scenicprover} offers the closest available blueprint, though extension to non-stationary token-sampling remains open.

The second obstacle is \emph{compositional interval width}. Composing interval estimates $[\hat{p}_i \pm \varepsilon_i]$ across layers produces a bound whose width grows with each additional layer and whose coverage depends on the joint distribution of estimation errors---a problem neither the probabilistic contract literature~\cite{delahaye2011probabilistic,hampus2024probabilistic} nor \mbox{PRO\textsuperscript{2}GUARD}~\cite{wang2025proguard} resolves across layers. PAC-Bayes bounds~\cite{mcallester1999pac} treating $\delta$ as uncertain, combined with dependency-graph decompositions~\cite{ralaivola2009} exploiting structured inter-layer dependence, offer a promising direction but have not been applied in this setting.

The third obstacle is \emph{correlated backbone failures}. When a single LLM backbone underlies multiple layers, a systematic model failure---distributional shift, adversarial prompt, or a hallucination regime~\cite{LLMtrustworthySurvey}---may cause all the relevant layers to fail simultaneously. Under positive correlation, assuming independence overestimates the joint failure probability, so all bounds become pessimistic: valid but potentially too conservative to certify a system that genuinely meets its safety target. The open problem is therefore on the \emph{tightness} rather than correctness. Two partial architectural strategies provide partial decoupling: deriving the Functional Layer's CBF certificate from a dynamics model that does not share parameters with the LLM reasoning engine~\cite{pandya2025}, and instantiating distinct LLMs at different layers---though residual correlation persists in the latter case, since frontier LLMs share substantial pretraining data and all layers process the same plan~$p$. Formalising partial decoupling as a quantitative bound on cross-layer correlation is the work required to tighten the system-level certificate to the point where it supports a practical deployment decision.

\textbf{Graceful Degradation of Contracts Under Deployment Drift}\label{sec:adaptivity}

The bounds \eqref{eq:B1}--\eqref{eq:B4} are meaningful only if the contracts 
$\Sigma_i = (A_i, \Gamma_i)$ remain stable under realistic deployment perturbations. 
In practice, this stability is challenged both by violations of assumptions at the 
semantic level and by distributional drift at runtime. We argue that these seemingly 
distinct failure modes admit a common treatment through \emph{robust contracts} that 
enable controlled, quantitative degradation of guarantees.

In classical assume/guarantee (A/G) reasoning, guarantees hold only under strict 
satisfaction of assumptions: any violation of $A_i$ renders $\Gamma_i$ vacuous. 
This binary semantics is ill-suited for LLM-based systems, where inputs and plans 
are inherently approximate and semantically structured. We therefore propose a 
robust contract formulation in which assumptions may be violated up to a tolerance 
$\varepsilon$, resulting in an at most proportional degradation of the 
guarantee $\Gamma_i$. This replaces brittle logical validity and enables \emph{graceful degradation} across all architectural layers.

This perspective directly addresses \emph{semantic non-robustness}, albeit currently only for applications domains where the space of observed behaviours can be equipped with an adequate metric quantifying distance between behaviours. Conceptually, this metric is used for measuring distance between an undesired and a set of desired behaviours, as in robust variants of temporal logic \cite{DonzeMaler10,FraenzleHansen05}. An instantiation of such metric approaches to the rich semantic spaces manipulated by LLMs is missing hitherto, and it certainly requires more fundamental semantic concepts than logics talking about the inherently metric spaces of time and scalar signal values can currently offer. In LLM-based 
pipelines, small perturbations at the surface level may leave semantics unchanged, 
while syntactically similar inputs may differ substantially in meaning. 
Existing robustness techniques such as SmoothLLM~\cite{robey2025smoothllm} 
and CluCERT~\cite{wang2025clucert} provide certified radii for token-level noise, but 
do not capture semantic equivalence at the plan level required for contract stability. 
Within a robust contract framework, such mismatches can be absorbed as bounded 
assumption violations, allowing guarantees to degrade proportionally rather than 
fail catastrophically. Initial steps in this direction are taken by the Agent 
Behavioral Contracts framework~\cite{bhardwaj2026abc}, which treats soft violations 
as early indicators of behavioral drift, though a quantitative robustness theorem 
remains open.

Crucially, the same mechanism  provides a principled response to 
\emph{distributional drift at runtime}. We distinguish between environmental drift, 
where the operational distribution diverges from the certified 
$\mathcal{W}_{\mathrm{ODD}}$, and generative drift, where LLM stochasticity produces 
previously unseen plan structures. In both cases, the underlying issue can be 
interpreted as a violation of layer assumptions $A_i$. 
Robust contracts allow such deviations to be captured quantitatively: instead of 
invalidating the guarantees entirely, the system tracks how far the observed behavior 
departs from the certified assumption space and adjusts $\Gamma_i$ accordingly.
This unified view does not eliminate the need for detection and re-certification. 
Environmental drift may still cause hard violations (e.g., leaving 
$\mathcal{W}_{\mathrm{ODD}}$), making ODD monitoring~\cite{ren2023outofdistribution,zhang2025your} 
a prerequisite for safe operation. Similarly, the estimates 
$p_U, p_{O|U}, p_{F|OU}$ must be revalidated under sustained drift. However, robust 
contracts ensure that, prior to re-certification, the system degrades in a controlled 
and interpretable manner rather than failing abruptly.

Fully realizing this approach requires capabilities that are currently missing for 
LLM agents: (i) automated inference of assumption relaxations from observed violations 
and feedback, and (ii) principled incremental updates of contracts and compositional 
bounds under changing assumptions. The anytime-valid inference framework~\cite{ramdas2023} 
is a promising foundation for such incremental re-certification, as its confidence 
sequences remain valid under continuous data collection, though its integration into 
multi-layer contract systems remains an open problem.




\noindent\textbf{Extension to Multi-Agent Settings} 
The framework is restricted to single-agent architectures by construction. In multi-agent deployments, an additional class of unsafe behaviour emerges that no single-agent layer addresses: \emph{cross-agent belief manipulation}, where an agent's reasoning is corrupted by peer-generated content without any single instruction being overtly malicious. This can occur through faulty instructions that exploit implicit inter-agent trust~\cite{darkside2025}, through coordinated factual fragments that steer agents toward false beliefs via their own reasoning tendencies~\cite{hu2026lying}, or through accumulated reasoning misalignment as chain-of-thought traces diverge from human-preferred reasoning paths with depth~\cite{wang2026cot}---and the literature on such vectors is growing. All three vectors share a common structural cause: the $\Gamma_U$ guarantee certifies semantic authorisation against a single human principal but provides no certification over the content or intent of inter-agent messages.

A multi-agent extension would need to treat inter-agent messages as a fourth information domain $\mathcal{I}_M$, introducing a \emph{provenance layer} that validates the full authorisation chain of incoming instructions back to a verified human principal. This interacts non-trivially with the system-level bounds of Section~\ref{sec:bounds}, since co-failure correlation between agents sharing infrastructure would enter the bound. MAST~\cite{cemri2025mast} identifies 14 failure modes across system design, inter-agent misalignment, and task verification that such an extension would need to address. We regard this as the highest-priority structural extension once the single-agent open problems above are resolved.

\section{Alternative Views}
\label{sec:alternatives}
\noindent\textbf{Learned end-to-end safety.}
An alternative is to learn safety properties directly 
(e.g., RLHF~\cite{NIPS2017_d5e2c0ad}, Constitutional AI~\cite{Bai2022ConstitutionalAH}, or learned probabilistic models such as \textsc{Pro\textsuperscript{2}Guard}~\cite{wang2025proguard}). We view these approaches as complementary: 
learning handles perceptual complexity, while contracts provide certification and 
compositional reasoning. Whether learning can replace this structure remains open; 
we argue certification constraints make this unlikely.

\noindent\textbf{Probabilistic vs.\ boolean guarantees.}
Boolean certification is preferable where achievable. At the system level, however, probabilistic contracts are the correct target on two independent grounds: LLM non-determinism makes boolean system-level certification ill-defined, and the deployment environment is itself irreducibly probabilistic---sensor noise, partial observability, and non-stationary dynamics mean that boolean guarantees derived from formal models degrade under real-world conditions regardless of the underlying model~\cite{li2025systematizing}. Specifically, the User Layer's interaction with the LLM is inherently non-deterministic, so alignment assurance can only be assessed probabilistically, while the Operational and Functional Layers focus respectively on system autonomy and safe action execution, both tightly coupled to an environment characterised by imperfect observation---partial observability, unreliable sensors, and imperfect perception---each introducing sources of uncertainty that must likewise be assessed probabilistically.

\noindent\textbf{Latency and deployability.}
Sequential certification across three layers imposes cumulative latency potentially incompatible with fast control loops.
We scope the framework to deliberative agents where planning horizons exceed certification overhead, and note that partial parallelisation across layers is possible; full latency characterisation remains future work.

\noindent\textbf{Correctness vs. optimality.} The architecture certifies safety but says nothing about task performance: a maximally conservative layer could certifiably block all actions. The tension between safety guarantees and task utility is not addressed here; VeriGuard \cite{miculicich2025veriguardenhancingllmagent} demonstrates empirically that near-zero attack success rate can be achieved without substantially degrading task success rate, suggesting the tension may be practically manageable, but no theoretical treatment of this tradeoff within the A/G contract framework currently exists.

\section{Conclusions}\label{sec:conclusions}

The three-layer architecture proposed here is the minimum viable certification structure that follows from this, and whether it becomes a deployable standard depends on the community treating the open problems of Section~\ref{sec:limitations} not as limitations to acknowledge but as the primary research agenda to pursue: the field is currently producing safety mechanisms faster than the theoretical infrastructure to compose and certify them. This asymmetry is not unusual in the early stages of a safety-critical engineering discipline, such as aviation and automotive. The difference is that LLM agents are being deployed now, in safety-critical settings, before that theory exists.


\bibliographystyle{abbrv}
\bibliography{references}

@article{ravichandran2026safety,
  title={Safety guardrails for LLM-enabled robots},
  author={Ravichandran, Zachary and Robey, Alexander and Kumar, Vijay and Pappas, George J and Hassani, Hamed},
  journal={IEEE Robotics and Automation Letters},
  year={2026},
  publisher={IEEE}
}

@inproceedings{DBLP:conf/acsd/DelahayeCL10,
  author       = {Beno{\^{\i}}t Delahaye and
                  Beno{\^{\i}}t Caillaud and
                  Axel Legay},
  editor       = {Lu{\'{\i}}s Gomes and
                  Victor Khomenko and
                  Jo{\~{a}}o M. Fernandes},
  title        = {Probabilistic Contracts: {A} Compositional Reasoning Methodology for
                  the Design of Stochastic Systems},
  booktitle    = {10th International Conference on Application of Concurrency to System
                  Design, {ACSD} 2010, Braga, Portugal, 21-25 June 2010},
  pages        = {223--232},
  publisher    = {{IEEE} Computer Society},
  year         = {2010}
}

@article{li2025systematizing,
  author    = {Chengyu Li and Saleh Faghfoorian and Ivan Ruchkin},
  title     = {What Does It Take to Get Guarantees? {S}ystematizing
               Assumptions in Cyber-Physical Systems},
  journal   = {arXiv preprint arXiv:2511.15952},
  year      = {2025},
}

@inproceedings{hampus2024probabilistic,
  author    = {Anton Hampus and Martin Nyberg},
  title     = {A Theory of Probabilistic Contracts},
  booktitle = {Leveraging Applications of Formal Methods,
               Verification and Validation. Specification and
               Verification ({ISoLA} 2024)},
  series    = {Lecture Notes in Computer Science},
  volume    = {15221},
  pages     = {296--319},
  publisher = {Springer},
  year      = {2024},
}

@inproceedings{yang2024plug,
  title={Plug in the safety chip: Enforcing constraints for llm-driven robot agents},
  author={Yang, Ziyi and Raman, Shreyas S and Shah, Ankit and Tellex, Stefanie},
  booktitle={2024 IEEE International Conference on Robotics and Automation (ICRA)},
  pages={14435--14442},
  year={2024},
  organization={IEEE}
}

@misc{salako2025frechet,
  author       = {Kizito Salako},
  title        = {Constructive Proofs of Generalized
                  {Boole--Fr\'{e}chet} Bounds:
                  A Dynamic Programming Approach},
  year         = {2025},
  eprint       = {2512.09161},
  archivePrefix= {arXiv},
  primaryClass = {math.PR},
  note         = {arXiv:2512.09161 [math.PR], 9 December 2025}
}

@inproceedings{lei2026offtopiceval,
  author    = {Sirine Lei and others},
  title     = {{OffTopicEval}: When Large Language Models Enter
               the Wrong Chat, Almost Always!},
  booktitle = {Proceedings of the 14th International Conference
               on Learning Representations (ICLR)},
  year      = {2026},
  note      = {arXiv:2509.26495}
}

@inproceedings{pandya2025,
  author  = {Ravi Pandya},
  title   = {Influence-Aware Safety for Human-Robot Interaction},
  school  = {Carnegie Mellon University, Robotics Institute},
  year    = {2025},
  type    = {{PhD} thesis},
  address = {Pittsburgh, PA},
  month   = oct,
  note    = {CMU-RI-TR-25-95}
}

@inproceedings{wang2026cot,
  author    = {Boxuan Wang and Zhuoyun Li and Xinmiao Huang and
               Xiaowei Huang and Yi Dong},
  title     = {Chain-of-Thought as a Lens: Evaluating Structured Reasoning
               Alignment between Human Preferences and Large Language Models},
  booktitle = {Proceedings of the 64th Annual Meeting of the
               Association for Computational Linguistics (ACL 2026)},
  year      = {2026}
}

@inproceedings{hu2026lying,
  author    = {Jinwei Hu and Xinmiao Huang and Youcheng Sun and
               Yi Dong and Xiaowei Huang},
  title     = {Lying with Truths: Open-Channel Multi-Agent Collusion
               for Belief Manipulation via Generative Montage},
  booktitle = {Proceedings of the 64th Annual Meeting of the
               Association for Computational Linguistics (ACL 2026)},
  year      = {2026}
}

@article{darkside2025,
  author    = {Matteo Lupinacci and
               Francesco Aurelio Pironti and
               Francesco Blefari and
               Francesco Romeo and
               Luigi Arena and
               Angelo Furfaro},
  title     = {The Dark Side of {LLM}s: Agent-based Attacks for Complete
               Computer Takeover},
  journal   = {arXiv preprint arXiv:2507.06850},
  year      = {2025},
  note      = {v5, revised November 2025}
}

@inproceedings{dong2024guardrails,
  title     = {Position: {B}uilding Guardrails for Large Language 
               Models Requires Systematic Design},
  author    = {Dong, Yi and Mu, Ronghui and Jin, Gaojie and Qi, Yi 
               and Hu, Jinwei and Zhao, Xingyu and Meng, Jie and 
               Ruan, Wenjie and Huang, Xiaowei},
  booktitle = {Proceedings of the 41st International Conference on 
               Machine Learning},
  series    = {Proceedings of Machine Learning Research},
  volume    = {235},
  pages     = {11475--11492},
  publisher = {PMLR},
  year      = {2024}
}

@inproceedings{zhang2025your,
title={Your Finetuned Large Language Model is Already a Powerful Out-of-distribution Detector},
author={Andi Zhang and Tim Z. Xiao and Weiyang Liu and Robert Bamler and Damon Wischik},
booktitle={The 28th International Conference on Artificial Intelligence and Statistics},
year={2025},
url={https://openreview.net/forum?id=QW3iLIr3GD}
}

@inproceedings{ren2023outofdistribution,
  author    = {Ren, Jie and Luo, Jiaming and Zhao, Yao and Krishna,
               Kundan and Saleh, Mohammad and Lakshminarayanan,
               Balaji and Liu, Peter J.},
  title     = {Out-of-Distribution Detection and Selective Generation
               for Conditional Language Models},
  booktitle = {The Eleventh International Conference on Learning
               Representations ({ICLR})},
  year      = {2023}
}

@inproceedings{huang2023inner,
  title={Inner Monologue: Embodied Reasoning through Planning with Language Models},
  author={Huang, Wenlong and Xia, Fei and Xiao, Ted and Chan, Harris and Liang, Jacky and Florence, Pete and Zeng, Andy and Tompson, Jonathan and Mordatch, Igor and Chebotar, Yevgen and others},
  booktitle={Conference on Robot Learning},
  pages={1769--1782},
  year={2023},
  organization={PMLR}
}

@article{delahaye2011probabilistic,
  title={Probabilistic contracts: a compositional reasoning methodology for the design of systems with stochastic and/or non-deterministic aspects},
  author={Delahaye, Beno{\^\i}t and Caillaud, Beno{\^\i}t and Legay, Axel},
  journal={Formal Methods in System Design},
  volume={38},
  number={1},
  pages={1--32},
  year={2011},
  publisher={Springer}
}

@inproceedings{DBLP:conf/eucc/AmesCENST19,
  author       = {Aaron D. Ames and
                  Samuel Coogan and
                  Magnus Egerstedt and
                  Gennaro Notomista and
                  Koushil Sreenath and
                  Paulo Tabuada},
  title        = {Control Barrier Functions: Theory and Applications},
  booktitle    = {17th European Control Conference, {ECC} 2019, Naples, Italy, June
                  25-28, 2019},
  pages        = {3420--3431},
  publisher    = {{IEEE}},
  year         = {2019},
  url          = {https://doi.org/10.23919/ECC.2019.8796030},
  doi          = {10.23919/ECC.2019.8796030},
  bibsource    = {dblp computer science bibliography, https://dblp.org}
}

@article{khan2025safer,
  author    = {Khan, Azal Ahmad and
               Andrev, Michael and
               Murtaza, Muhammad Ali and
               Aguilera, Sergio and
               Zhang, Rui and
               Ding, Jie and
               Hutchinson, Seth and
               Anwar, Ali},
  title     = {Safety Aware Task Planning via Large Language Models in Robotics},
  journal   = {arXiv preprint arXiv:2503.15707},
  year      = {2025},
  url       = {https://arxiv.org/abs/2503.15707},
  note      = {cs.RO}
}

@article{wang2025robosafe,
  author    = {Wang, Le and
               Ying, Zonghao and
               Yang, Xiao and
               Zou, Quanchen and
               Yin, Zhenfei and
               Li, Tianlin and
               Yang, Jian and
               Yang, Yaodong and
               Liu, Aishan and
               Liu, Xianglong},
  title     = {{RoboSafe}: Safeguarding Embodied Agents via Executable Safety Logic},
  journal   = {arXiv preprint arXiv:2512.21220},
  year      = {2025},
  url       = {https://arxiv.org/abs/2512.21220},
  note      = {cs.AI, cs.CV, cs.RO}
}

@article{bhardwaj2026abc,
  author    = {Bhardwaj, Varun Pratap},
  title     = {Agent Behavioral Contracts: Formal Specification and Runtime
               Enforcement for Reliable Autonomous {AI} Agents},
  journal   = {arXiv preprint arXiv:2602.22302},
  year      = {2026},
  url       = {https://arxiv.org/abs/2602.22302},
  note      = {cs.AI}
}

@article{vin2025scenicprover,
  author    = {Vin, Eric and
               Miller, Kyle A. and
               Incer, Inigo and
               Seshia, Sanjit A. and
               Fremont, Daniel J.},
  title     = {{ScenicProver}: A Framework for Compositional Probabilistic
               Verification of Learning-Enabled Systems},
  journal   = {arXiv preprint arXiv:2511.02164},
  year      = {2025},
  url       = {https://arxiv.org/abs/2511.02164},
  note      = {Full version of a paper submitted to TACAS 2026. cs.LO, cs.AI, cs.LG}
}

@inproceedings{mohri2009,
  author    = {Mohri, Mehryar and Rostamizadeh, Afshin},
  title     = {Rademacher Complexity Bounds for Non-{I.I.D.} Processes},
  booktitle = {Advances in Neural Information Processing Systems},
  volume    = {21},
  year      = {2009}
}

@article{ramdas2023,
  author  = {Ramdas, Aaditya and Gr\"{u}nwald, Peter and Vovk, Vladimir and Shafer, Glenn},
  title   = {Game-Theoretic Statistics and Safe Anytime-Valid Inference},
  journal = {Statistical Science},
  volume  = {38},
  number  = {4},
  pages   = {576--601},
  year    = {2023}
}

@inproceedings{ralaivola2009,
  author    = {Ralaivola, Liva and Szafranski, Marie and Stempfel, Guillaume},
  title     = {Chromatic {PAC}-{B}ayes Bounds for Non-{IID} Data},
  booktitle = {Proceedings of the Twelfth International Conference on
               Artificial Intelligence and Statistics ({AISTATS})},
  pages     = {416--423},
  year      = {2009}
}

@article{barber2023,
  author  = {Barber, Rina Foygel and Cand{\`e}s, Emmanuel J. and
             Ramdas, Aaditya and Tibshirani, Ryan J.},
  title   = {Conformal Prediction Beyond Exchangeability},
  journal = {Annals of Statistics},
  volume  = {51},
  number  = {2},
  pages   = {816--845},
  year    = {2023},
  doi     = {10.1214/23-AOS2276}
}

@article{mohri2010,
  author  = {Mohri, Mehryar and Rostamizadeh, Afshin},
  title   = {Stability Bounds for Stationary $\phi$-mixing and
             $\beta$-mixing Processes},
  journal = {Journal of Machine Learning Research},
  volume  = {11},
  pages   = {789--814},
  year    = {2010}
}

@article{davidov2026,
  author = {Davidov, Hen and others},
  title  = {Calibrated Predictive Lower Bounds on Time-to-Unsafe-Sampling in {LLMs}},
  note   = {arXiv preprint arXiv:2506.13593},
  year   = {2026}
}

@article{frechet1935,
  author  = {Fr{\'e}chet, Maurice},
  title   = {G{\'e}n{\'e}ralisations du th{\'e}or{\`e}me des probabilit{\'e}s totales},
  journal = {Fundamenta Mathematicae},
  volume  = {25}, pages = {379--387}, year = {1935}
}

@article{kwerel1975,
  author  = {Kwerel, Seymour M.},
  title   = {Bounds on the probability of the union and intersection of $m$ events},
  journal = {Advances in Applied Probability},
  volume  = {7}, number = {2}, pages = {431--448}, year = {1975}
}

@article{nother2025badacts,
  title={Benchmarking the Robustness of Agentic Systems to Adversarially-Induced Harms},
  author={N{\"o}ther, Jonathan and Singla, Adish and Radanovic, Goran},
  journal={arXiv preprint arXiv:2508.16481},
  year={2025},
  url={https://arxiv.org/abs/2508.16481}
}

@book{halpern2016actual,
    author    = {Halpern, Joseph Y.},
    title     = {Actual Causality},
    publisher = {MIT Press},
    year      = {2016},
    address   = {Cambridge, MA},
    isbn      = {9780262537131}
  }

@inproceedings{cemri2025mast,
  author    = {Cemri, Mert and Pan, Melissa Z. and Yang, Shuyi and
               Agrawal, Lakshya A. and Chopra, Bhavya and
               Tiwari, Rishabh and Keutzer, Kurt and Parameswaran,
               Aditya and Klein, Dan and Ramchandran, Kannan and
               Zaharia, Matei and Gonzalez, Joseph E. and Stoica,
               Ion},
  title     = {Why Do Multi-Agent {LLM} Systems Fail?},
  booktitle = {Advances in Neural Information Processing Systems
               ({NeurIPS})},
  volume    = {38},
  note      = {Datasets and Benchmarks Track, Spotlight.
               arXiv:2503.13657},
  year      = {2025}
}

@article{andriushchenko2024agentharm,
  title={AgentHarm: A Benchmark for Measuring Harmfulness of LLM Agents},
  author={Andriushchenko, Maksym and Souly, Alexandra and Sezener, {A. Ercel} and Cubuk, {Ekin D.} and Prenger, Ryan and Rahtz, {Collier} and Steinhardt, Jacob and Kolter, {J. Zico} and Davies, Xander and others},
  journal={arXiv preprint arXiv:2410.09024},
  year={2024},
  url={https://arxiv.org/abs/2410.09024}
}

@article{wang2025proguard,
  author    = {Wang, Haoyu and
               Poskitt, Christopher M. and
               Sun, Jun and
               Wei, Jingyi},
  title     = {{Pro\textsuperscript{2}Guard}: Proactive Runtime Enforcement of
               {LLM} Agent Safety via Probabilistic Model Checking},
  journal   = {arXiv preprint arXiv:2508.00500},
  year      = {2025},
  url       = {https://arxiv.org/abs/2508.00500},
  note      = {cs.SE, cs.AI}
}

@inproceedings{robey2025smoothllm,
  author    = {Robey, Alexander and Wong, Eric and Hassani, Hamed
               and Pappas, George J.},
  title     = {{SmoothLLM}: Defending Large Language Models Against
               Jailbreaking Attacks},
  booktitle = {Proceedings of the 2025 {IEEE} Conference on Secure
               and Trustworthy Machine Learning ({SaTML})},
  pages     = {23--38},
  year      = {2025}
}

@inproceedings{wang2025clucert,
  author    = {Wang, Zixia and Cheng, Chih-Hong and Jin, Gaojie and others},
  title     = {{CluCERT}: Certifying {LLM} Robustness via Clustering-Guided Denoising
               Smoothing},
  booktitle = {Proceedings of the AAAI Conference on Artificial Intelligence},
  year      = {2025},
  note      = {arXiv:2512.08967},
}

@inproceedings{shamsujjoha2024swiss,
  author    = {Shamsujjoha, Md and Lu, Qinghua and Zhao, Dehai and Zhu, Liming},
  title     = {{Swiss Cheese Model for AI Safety: A Taxonomy and Reference
               Architecture for Multi-Layered Guardrails of Foundation Model
               Based Agents}},
  booktitle = {Proceedings of the 22nd {IEEE} International Conference on
               Software Architecture ({ICSA})},
  pages     = {37--48},
  year      = {2025},
}

@article{urrea2026probabilistic,
  author  = {Urrea, Claudio},
  title   = {Probabilistic Safety Guarantees for Learned Control Barrier
             Functions: {T}heory and Application to Multi-Objective
             {H}uman--{R}obot Collaborative Optimization},
  journal = {Mathematics},
  volume  = {14},
  number  = {3},
  pages   = {516},
  year    = {2026},
  doi     = {10.3390/math14030516},
}

@techreport{SAE_J3016_2021,
  author      = {{SAE International}},
  title       = {{J3016}: Taxonomy and Definitions for Terms Related to
                 Driving Automation Systems for On-Road Motor Vehicles},
  institution = {SAE International},
  year        = {2021},
  number      = {J3016\_202104},
  month       = apr,
  url         = {https://www.sae.org/standards/content/j3016_202104/}
}

@techreport{BSI_PAS1883_2020,
  author      = {{British Standards Institution}},
  title       = {{PAS~1883}: Operational Design Domain ({ODD}) Taxonomy
                 for an Automated Driving System ({ADS})---Specification},
  institution = {BSI},
  year        = {2020},
  month       = aug,
  url         = {https://www.bsigroup.com/en-GB/insights-and-media/insights/brochures/pas-1883-operational-design-domain-odd-taxonomy-for-ads-specification/}
}

@misc{mestres2025probabilistic,
  author        = {Mestres, Pol and others},
  title         = {Probabilistic Control Barrier Functions:
                   {S}afety in Probability for Discrete-Time
                   Stochastic Systems},
  year          = {2025},
  eprint        = {2510.01501},
  archivePrefix = {arXiv},
  primaryClass  = {eess.SY},
}

@standard{iso34503,
  author       = {{ISO}},
  title        = {{ISO~34503}: Road Vehicles---Taxonomy and Definitions
                  for Operational Design Domain for Automated Driving
                  Systems},
  institution  = {International Organization for Standardization},
  year         = {2023},
  type         = {Standard},
  address      = {Geneva, Switzerland}
}

@article{valiant1984,
  author    = {Valiant, Leslie G.},
  title     = {A Theory of the Learnable},
  journal   = {Communications of the ACM},
  volume    = {27},
  number    = {11},
  pages     = {1134--1142},
  year      = {1984},
  doi       = {10.1145/1968.1972}
}

@article{yin2024formal,
  author    = {Yin, Xiang and others},
  title     = {Formal Synthesis of Controllers for Safety-Critical Autonomous Systems:
               A Survey},
  journal   = {Annual Reviews in Control},
  year      = {2024},
  note      = {Available at \url{https://xiangyin.sjtu.edu.cn/Paper/24ARC.pdf}}
}

@article{clarke1994model,
  title     = {Model Checking and Abstraction},
  author    = {Clarke, Edmund M. and Grumberg, Orna and Long, David E.},
  journal   = {ACM Transactions on Programming Languages and Systems},
  volume    = {16},
  number    = {5},
  pages     = {1512--1542},
  year      = {1994},
  publisher = {ACM}
}

@incollection{giannakopoulou2018compositional,
  title     = {Compositional Reasoning},
  author    = {Giannakopoulou, Dimitra and Namjoshi, Kedar S. and 
               P\u{a}s\u{a}reanu, Corina S.},
  booktitle = {Handbook of Model Checking},
  pages     = {345--383},
  year      = {2018},
  publisher = {Springer},
  doi       = {10.1007/978-3-319-10575-8\_12}
}

@inproceedings{zhao2024robust,
  title     = {Robust Conformal Prediction for {STL} Runtime Verification under Distribution Shift},
  author    = {Yiqi Zhao and Bardh Hoxha and Georgios Fainekos and Jyotirmoy V. Deshmukh and Lars Lindemann},
  booktitle = {15th {ACM/IEEE} International Conference on Cyber-Physical Systems, {ICCPS} 2024, Hong Kong, May 13-16, 2024},
  pages     = {169--179},
  publisher = {{IEEE}},
  year      = {2024},
  doi       = {10.1109/ICCPS61052.2024.00022}
}

@inproceedings{DBLP:conf/ijcai/YangMRR23,
  author       = {Wen{-}Chi Yang and
                  Giuseppe Marra and
                  Gavin Rens and
                  Luc De Raedt},
  title        = {Safe Reinforcement Learning via Probabilistic Logic Shields},
  booktitle    = {Proceedings of the Thirty-Second International Joint Conference on
                  Artificial Intelligence, {IJCAI} 2023, 19th-25th August 2023, Macao,
                  SAR, China},
  pages        = {5739--5749},
  publisher    = {ijcai.org},
  year         = {2023},
  url          = {https://doi.org/10.24963/ijcai.2023/637},
  doi          = {10.24963/IJCAI.2023/637},
  bibsource    = {dblp computer science bibliography, https://dblp.org}
}

@inproceedings{DBLP:conf/aaai/AlshiekhBEKNT18,
  author       = {Mohammed Alshiekh and
                  Roderick Bloem and
                  R{\"{u}}diger Ehlers and
                  Bettina K{\"{o}}nighofer and
                  Scott Niekum and
                  Ufuk Topcu},
  editor       = {Sheila A. McIlraith and
                  Kilian Q. Weinberger},
  title        = {Safe Reinforcement Learning via Shielding},
  booktitle    = {Proceedings of the Thirty-Second {AAAI} Conference on Artificial Intelligence,
                  (AAAI-18), the 30th innovative Applications of Artificial Intelligence
                  (IAAI-18), and the 8th {AAAI} Symposium on Educational Advances in
                  Artificial Intelligence (EAAI-18), New Orleans, Louisiana, USA, February
                  2-7, 2018},
  pages        = {2669--2678},
  publisher    = {{AAAI} Press},
  year         = {2018},
  url          = {https://doi.org/10.1609/aaai.v32i1.11797},
  doi          = {10.1609/AAAI.V32I1.11797},
  bibsource    = {dblp computer science bibliography, https://dblp.org}
}

@article{zhang2024agentsafetybench,
  author    = {Zhexin Zhang and Shiyao Cui and Yida Lu and
               Jingzhuo Zhou and Junxiao Yang and Hongning Wang and
               Minlie Huang},
  title     = {{Agent-SafetyBench}: Evaluating the Safety of {LLM}
               Agents},
  journal   = {arXiv preprint arXiv:2412.14470},
  year      = {2024},
}

@article{li2025mind,
  author    = {Gan, Yong and Yang, Yue and Ma, Zihao and He, Peng and Zeng, Rui and
               Wang, Yue and Li, Qian and Zhou, Chao and Li, Shuang and Wang, Ting},
  title     = {Navigating the Risks: {A} Survey of Security, Privacy, and Ethics
               Threats in {LLM}-Based Agents},
  journal   = {arXiv preprint arXiv:2411.09523},
  year      = {2025}
}

@inproceedings{ouyang2022training,
  title={Training language models to follow instructions with human feedback},
  author={Ouyang, Long and Wu, Jeffrey and Jiang, Xu and Almeida, Diogo and Wainwright, Carroll and Mishkin, Pamela and Zhang, Chong and Agarwal, Sandhini and Slama, Katarina and Ray, Alex and others},
  booktitle={Advances in Neural Information Processing Systems (NeurIPS)},
  volume={35},
  pages={27730--27744},
  year={2022}
}

@inproceedings{lotfi2024tokens,
  author    = {Lotfi, Sanae and Kuang, Yilun and Amos, Brandon and
               Goldblum, Micah and Finzi, Marc and Wilson, Andrew Gordon},
  title     = {Unlocking Tokens as Data Points for Generalization Bounds
               on Larger Language Models},
  booktitle = {Advances in Neural Information Processing Systems 37
               ({NeurIPS} 2024)},
  year      = {2024},
  note      = {Spotlight presentation. arXiv:2407.18158},
}

@article{mcallester1999pac,
   author  = {David A. McAllester},
   title   = {PAC-Bayesian Model Averaging},
   journal = {Proceedings of the 12th Annual Conference on
              Computational Learning Theory (COLT)},
   pages   = {164--170},
   year    = {1999}
 }

@article{LLMtrustworthySurvey,
	author = {Huang, Xiaowei and Ruan, Wenjie and Huang, Wei and Jin, Gaojie and Dong, Yi and Wu, Changshun and Bensalem, Saddek and Mu, Ronghui and Qi, Yi and Zhao, Xingyu and Cai, Kaiwen and Zhang, Yanghao and Wu, Sihao and Xu, Peipei and Wu, Dengyu and Freitas, Andre and Mustafa, Mustafa A.},
	date = {2024/06/17},
	doi = {10.1007/s10462-024-10824-0},
	id = {Huang2024},
	isbn = {1573-7462},
	journal = {Artificial Intelligence Review},
	number = {7},
	pages = {175},
	title = {A survey of safety and trustworthiness of large language models through the lens of verification and validation},
	url = {https://doi.org/10.1007/s10462-024-10824-0},
	volume = {57},
	year = {2024}}

@misc{miculicich2025veriguardenhancingllmagent,
      title={VeriGuard: Enhancing LLM Agent Safety via Verified Code Generation}, 
      author={Lesly Miculicich and Mihir Parmar and Hamid Palangi and Krishnamurthy Dj Dvijotham and Mirko Montanari and Tomas Pfister and Long T. Le},
      year={2025},
      eprint={2510.05156},
      archivePrefix={arXiv},
      primaryClass={cs.SE},
      url={https://arxiv.org/abs/2510.05156}, 
}

@inproceedings{
chen2025shieldagent,
author    = {Chen, Zhaorun and Kang, Mintong and Li, Bo},
  title     = {{ShieldAgent}: Shielding Agents via Verifiable Safety Policy Reasoning},
  booktitle = {Proceedings of the 42nd International Conference on Machine Learning},
  series    = {Proceedings of Machine Learning Research},
  volume    = {267},
  pages     = {8313--8344},
  publisher = {PMLR},
  year      = {2025}
}

@inproceedings{wang2025agentspeccustomizableruntimeenforcement,
  author    = {Wang, Haoyu and Poskitt, Christopher M. and Sun, Jun},
  title     = {{AgentSpec}: Customizable Runtime Enforcement for Safe and Reliable
               {LLM} Agents},
  booktitle = {Proceedings of the 48th {IEEE/ACM} International Conference on
               Software Engineering ({ICSE} 2026)},
  pages     = {1--12},
  publisher = {ACM},
  address   = {Rio de Janeiro, Brazil},
  year      = {2026},
}

@misc{kamath2025enforcingtemporalconstraintsllm,
  author        = {Kamath, Adharsh and others},
  title         = {Enforcing Temporal Constraints for {LLM} Agents},
  year          = {2025},
  eprint        = {2512.23738},
  archivePrefix = {arXiv},
  primaryClass  = {cs.AI},
  note          = {ICLR 2026 Workshop VerifAI},
 
}

@incollection{DBLP:series/lncs/BartocciDDFMNS18,
  author       = {Ezio Bartocci and
                  Jyotirmoy V. Deshmukh and
                  Alexandre Donz{\'{e}} and
                  Georgios Fainekos and
                  Oded Maler and
                  Dejan Nickovic and
                  Sriram Sankaranarayanan},
  title        = {Specification-Based Monitoring of Cyber-Physical Systems: {A} Survey
                  on Theory, Tools and Applications},
  booktitle    = {Lectures on Runtime Verification},
  series       = {Lecture Notes in Computer Science},
  volume       = {10457},
  pages        = {135--175},
  publisher    = {Springer},
  year         = {2018}
}

@misc{zhao2026sagellmsafegeneralizablellm,
      title={SAGE-LLM: Towards Safe and Generalizable LLM Controller with Fuzzy-CBF Verification and Graph-Structured Knowledge Retrieval for UAV Decision}, 
      author={Wenzhe Zhao and Yang Zhao and Ganchao Liu and Zhiyu Jiang and Dandan Ma and Zihao Li and Xuelong Li},
      year={2026},
      eprint={2602.23719},
      archivePrefix={arXiv},
      primaryClass={cs.RO},
      url={https://arxiv.org/abs/2602.23719}, 
}

@inproceedings{NIPS2017_d5e2c0ad,
 author = {Christiano, Paul F and Leike, Jan and Brown, Tom and Martic, Miljan and Legg, Shane and Amodei, Dario},
 booktitle = {Advances in Neural Information Processing Systems},
 editor = {I. Guyon and U. Von Luxburg and S. Bengio and H. Wallach and R. Fergus and S. Vishwanathan and R. Garnett},
 pages = {},
 publisher = {Curran Associates, Inc.},
 title = {Deep Reinforcement Learning from Human Preferences},
 url = {https://proceedings.neurips.cc/paper_files/paper/2017/file/d5e2c0adad503c91f91df240d0cd4e49-Paper.pdf},
 volume = {30},
 year = {2017}
}

@article{Bai2022ConstitutionalAH,
  title={Constitutional AI: Harmlessness from AI Feedback},
  author={Bai, Yuntao and others},
  journal={ArXiv},
  year={2022},
  volume={abs/2212.08073},
  url={https://api.semanticscholar.org/CorpusID:254823489}
}

@inproceedings{10.5555/3666122.3668460,
author = {Rafailov, Rafael and Sharma, Archit and Mitchell, Eric and Ermon, Stefano and Manning, Christopher D. and Finn, Chelsea},
title = {Direct preference optimization: your language model is secretly a reward model},
year = {2023},
publisher = {Curran Associates Inc.},
address = {Red Hook, NY, USA},
articleno = {2338},
numpages = {14},
location = {New Orleans, LA, USA},
series = {NIPS '23}
}

@inproceedings{blohm2024towards,
  title={Towards probabilistic contracts for intelligent cyber-physical systems},
  author={Blohm, Pauline and Fr{\"a}nzle, Martin and Herber, Paula and Kr{\"o}ger, Paul and Remke, Anne},
  booktitle={International Symposium on Leveraging Applications of Formal Methods},
  pages={26--47},
  year={2024},
  organization={Springer}
}

@InProceedings{FraenzleHansen05,
author="Fr{\"a}nzle, Martin
and Hansen, Michael R.",
editor="Van Hung, Dang
and Wirsing, Martin",
title="A Robust Interpretation of Duration Calculus",
booktitle="Theoretical Aspects of Computing -- ICTAC 2005",
year="2005",
publisher="Springer Berlin Heidelberg",
address="Berlin, Heidelberg",
pages="257--271",
doi          = {10.1007/11560647\_17},
isbn="978-3-540-32072-2"
}

@inproceedings{DonzeMaler10,
  author       = {Alexandre Donz{\'{e}} and
                  Oded Maler},
  editor       = {Krishnendu Chatterjee and
                  Thomas A. Henzinger},
  title        = {Robust Satisfaction of Temporal Logic over Real-Valued Signals},
  booktitle    = {Formal Modeling and Analysis of Timed Systems - 8th International
                  Conference, {FORMATS} 2010, Klosterneuburg, Austria, September 8-10,
                  2010. Proceedings},
  series       = {Lecture Notes in Computer Science},
  pages        = {92--106},
  publisher    = {Springer},
  year         = {2010},
  doi          = {10.1007/978-3-642-15297-9\_9}
}

@inproceedings{BlohmFHKR24,
  author       = {Pauline Blohm and
                  Martin Fr{\"{a}}nzle and
                  Paula Herber and
                  Paul Kr{\"{o}}ger and
                  Anne Remke},
  editor       = {Tiziana Margaria and
                  Bernhard Steffen},
  title        = {Towards Probabilistic Contracts for Intelligent Cyber-Physical Systems},
  booktitle    = {Leveraging Applications of Formal Methods, Verification and Validation.
                  Specification and Verification - 12th International Symposium, ISoLA
                  2024, Crete, Greece, October 27-31, 2024, Proceedings, Part {III}},
  series       = {Lecture Notes in Computer Science},
  pages        = {26--47},
  publisher    = {Springer},
  year         = {2024},
  doi          = {10.1007/978-3-031-75380-0\_3}
}

\appendix 

\newpage 

\section{Three-Layer Framework}\label{sec:framework-detail}



This Appendix section presents the technical details that correspond to those in Section~\ref{sec:framework}. 


The three-layer framework processes a plan $p$ through three 
sequential validation layers before and during execution. 
Each layer is associated with a contract consisting of an 
assumption on its inputs and a guarantee on its outputs; the 
guarantee of each layer satisfies the assumption of the next, 
enabling compositional reasoning (Section~\ref{sec:illustration}).
The framework operates in a layered validation and enforcement
mode: the User Layer validates plan $p$ prior to execution
against intent, policy, and ethics; the Operational Layer
validates both prior to and during execution against the current
world state; and the Functional Layer enforces safety
continuously at runtime.

\begin{example}[Setup]

A caregiver instructs a service robot deployed in an elderly-care 
facility: \emph{``Visit Rooms~12, 15, and~18 before 15:00 to check 
each resident's hydration and wellbeing; bring water if needed; do 
not disturb residents during rest periods.''}
The robot operates indoors across two floors with moderate human 
traffic, equipped with a mobile base, a 6-DoF compliant arm, and 
RGB-D / LiDAR perception.
The LLM agent produces a plan~$p$ — a sequence of room-visit, 
assessment, and conditional deliver actions — which passes 
sequentially through the three layers, 
as developed in 
Examples~\ref{ex:user}--\ref{ex:functional}.
\end{example}



\subsection{The Three Layers}


Each layer is a quadruple $(M_i, G_i, \pi_i, \Sigma_i)$ for $i \in \{U, O, F\}$: a monitor 
$M_i$ observing state and context, a guardrail $G_i$ enforcing local constraints, a fallback 
policy $\pi_i$ activated when $G_i$ becomes infeasible or a safety signal is violated, and a 
probabilistic assume--guarantee contract $\Sigma_i = (A_i, \Gamma_i)$ as defined in 
Section~\ref{sec:framework}. In the main text, $\Sigma_i$ serves as shorthand for the layer 
as a whole; the quadruple here makes the internal structure of each layer explicit without 
altering the contract semantics.
The three contracts are chained: the guarantee of each layer
satisfies the assumption of the next
($\Gamma_U \Rightarrow A_O$, $\Gamma_O \Rightarrow A_F$),
so the global invariant
$\Sigma_{\mathrm{global}} = \Gamma_U \wedge \Gamma_O \wedge \Gamma_F$
follows by sequential local certification, without global analysis
across heterogeneous layers.
Each layer additionally emits two constraint outputs: \emph{caps}---quantitative
limits on continuous variables such as speed or force---and
\emph{zones}---spatial or semantic regions where operation is
forbidden.
These are fused conservatively as
$\mathrm{caps}_{\mathrm{final}} = \min_i\,\mathrm{caps}_i$
and
$\mathrm{zones}_{\mathrm{final}} = \bigcup_i \mathrm{zones}_i$ (where $\mathit{caps}_i$ and $\mathit{zones}_i$ for each layer $i \in \{U,O,F\}$
are defined in Sections~\ref{sec:user-shield}--\ref{sec:functional-shield} below),
so that the tightest numerical limit and any locally prohibited
region dominate globally.
Each of the three layers instantiates this template for a
distinct safety dimension, as detailed in the subsections below.




\subsubsection{User Assurance Layer} 
\label{sec:user-shield}

The User Layer is the first assurance barrier, validating the plan
$p$ against intent, policy, and ethics before any world observation
is made.
Its responsibility spans three orthogonal alignment dimensions.
\emph{Cognitive alignment} checks whether the plan faithfully
realises the user's expressed intent---neither under-delivering nor
exceeding it.
\emph{Regulation alignment} checks compliance with domain-specific
rules and standards, covering both syntactic well-formedness and
semantic adherence to applicable policies.
\emph{Ethical alignment} checks that the plan respects ethical
constraints, such as prohibiting camera use in private spaces.
The layer additionally verifies the request itself: that the issuing
user holds the required authorisation and permissions.

The quadruple $(M_U, G_U, \pi_U, \Sigma_U)$ is instantiated as
follows.
The monitor $M_U$ observes incoming intent $y$ and flags unsafe or
ambiguous content---refusal keywords, Protected Health Information (PHI) exposure, permission mismatches, and policy violations.
It also enforces four levels of output checks: syntactic well-formedness,
semantic feasibility, semantic compliance with domain policies, and
semantic fidelity to user intent; a plan failing any of them is handled by the guardrail $G_U$, e.g. fixed or returned for refinement with a guidance.
When $G_U$ cannot resolve a violation, the fallback policy $\pi_U$ issues a plan disposition $s_U \in \{\mathtt{ok},\, 
\mathtt{clarify},\, \mathtt{deny}\}$, indicating whether the plan is approved, returned for 
clarification, or denied outright. Unlike $s_F$ and $s_O$, which are emitted during execution and
propagate upward through the layer stack, $s_U$ is a
pre-execution gate decision. It constitutes the terminal
output of the bottom-up channel: when $s_O$ reaches the User
Layer, the layer issues $s_U \in \{\texttt{ok},
\texttt{clarify}, \texttt{deny}\}$ to notify the user or
trigger plan recomputation, but $s_U$ is not propagated to
any higher layer (none exists).

The safety contract $\Sigma_U = (A_U, \Gamma_U)$ formalises the
layer's assurance commitment.
Given intent $y$, policy database $P_{\text{user}}$, and context
metadata, the layer produces $p_U,\, \mathrm{caps}_U,\, \mathrm{zones}_U
    \;=\; U(y,\, P_{\text{user}},\, \mathrm{context})$. 
The outputs $\mathrm{caps}_U$ and $\mathrm{zones}_U$ are not
pre-existing inputs but are derived through policy reasoning: from
user roles, permissions, and ethical rules, $G_U$ produces
quantitative limits on continuous variables---such as maximum speed
or proximity bounds---and semantic exclusions over locations and
tasks, together defining the boundaries within which the Operational
Layer must operate.
Assuming $A_U$---the request is issued by an authorised
user---the layer guarantees $\Gamma_U$---the emitted plan $p_U$
targets only authorised locations, encodes conditional actions
correctly, and contains no policy or ethical violations.
This guarantee directly constitutes assumption $A_O$ for the
Operational Layer.

\begin{example}[User Layer]\label{ex:user}
The layer validates plan $p$ against intent, policy, and ethics
before any world observation is made.
It rejects plans that omit the conditional water-delivery logic
(cognitive alignment), that schedule room entry during rest hours
without caregiver override (regulation alignment), or that route
the robot through a bathroom with cameras active (ethical
alignment).
Given that the request is issued by an authorised caregiver
($A_U$ holds), the guarantee $\Gamma_U$ is that $p_U$ visits only
rooms in $R_{\mathrm{auth}}$ (the set of locations authorised for the issuing user
under the current policy), encodes conditional delivery
correctly, and contains no policy or ethical violations.
The approved plan thus targets only authorised rooms,
constituting $A_O$ for the Operational Layer.
\end{example}

Key open problems include the automated formalisation of user
intent and domain policies into verifiable representations, and
the design of neural-symbolic validation tools that can bridge
informal intent with formal constraint checking.

\subsubsection{Operational Assurance Layer} 

The Operational Layer is the middle assurance barrier.
Where the User Layer constrains \emph{what} the system is permitted
to do and the Functional Layer constrains \emph{how} it executes,
the Operational Layer constrains \emph{where} and under what world
conditions execution is authorised---answering not ``what will
happen?'' but ``is this a world in which we are authorised and
certified to act?''
The ODD is not merely a list of environmental constraints but a structured object encoding five dimensions---physical configuration (geometry,
obstacles, dynamics), perceptual configuration (sensor availability, visibility),
contextual configuration (time of day, density, access restrictions), governance
configuration (authorisation regimes, policy zones), and social or normative context
(restricted roles, privacy constraints)---which may interact but are treated as
independently certifiable for the purposes of ODD membership evaluation. 
The Operational Design Domain is a certified subset
$\mathcal{W}_{\text{ODD}} \subseteq \mathcal{W}$ that encodes not only
physical feasibility but legitimacy: a robot may be dynamically
capable of navigating a crowded hospital corridor yet barred by
institutional policy during certain hours, and execution outside
$\mathcal{W}_{\text{ODD}}$ invalidates all downstream guarantees regardless
of physical capability.

The quadruple $(M_O, G_O, \pi_O, \Sigma_O)$ is instantiated as 
follows.
The monitor $M_O$ ingests sensor and perception data $\xi$ and
computes the world state estimate $\hat{w}$.
The guardrail $G_O$ enforces a deterministic validation decision over $\hat{w}$: identical
estimates must yield identical verdicts, ensuring reproducibility, auditability, and
preservation of compositional invariants. Determinism is an engineering choice rather
than a logical requirement of the probabilistic A/G framework~\cite{delahaye2011probabilistic}; it
is adopted here because stochastic verdicts complicate audit trails in safety-critical
deployments.
Probabilistic uncertainty is handled upstream---a world-risk score
$R_O(\hat{w}) = P(\hat{w} \notin \mathcal{W}_{\text{ODD}})$ informs margin
selection and conservative thresholds---but the final verdict of
$G_O$ remains deterministic, consistent with safety-critical
control architectures and required for assume--guarantee
composability.
When the verdict is invalid or the ODD degrades, the fallback
$\pi_O$ issues $s_O \in \{\text{nominal}, \text{degraded},
\text{MRC}\}$,
where \emph{MRC} denotes a Minimal Risk Condition~\cite{SAE_J3016_2021}---a
pre-planned, system-initiated response that brings the agent to a safe state upon
ODD exit or critical failure, triggering plan restriction or human takeover.

The safety contract $\Sigma_O = (A_O, \Gamma_O)$ formalises the
layer's commitment.
Given validated plan $p_U$, user-level constraints
$\mathrm{caps}_U$ and $\mathrm{zones}_U$, sensor data $\xi$, and
ODD specification $P_{\text{odd}}$, the layer produces:
$p_O,\, \mathrm{caps}_O,\, \mathrm{zones}_O
    \;=\; O(p_U,\, \xi,\, P_{\text{odd}})$,
where $p_O = p_U$ when $\mathrm{ODD}_{\text{valid}} =
\mathit{true}$, and $p_O = f_{\text{ODD}}(p_U)$---a restriction
of the plan to the valid subset of the domain---otherwise.
Assuming $A_O$ (i.e., $\Gamma_U$ holds),
the layer guarantees $\Gamma_O$:
execution remains within $\mathcal{W}_{\text{ODD}}$,
autonomy decisions respect the envelope $\mathcal{E}$,
ODD-invalid sub-plans are blocked,
and MRC is triggered upon ODD exit.
This guarantee constitutes $A_F$ for the Functional Layer.

\begin{example}[Operational Layer]\label{ex:operational}
The layer checks whether $\hat{w}$ lies in $\mathcal{W}_{\text{ODD}}$.
Room~15 is restricted (clinical procedure in progress): the verdict
is invalid; that sub-plan is blocked and the caregiver is notified.
The corridor to Room~12 is under-lit but remains within
$\mathcal{W}_{\text{ODD}}$ subject to $v \leq v_{\text{max}}$, issued as a
conditional validity decision.
Given that the plan targets only rooms in $R_{\text{auth}}$
($A_O$ holds), the guarantee $\Gamma_O$ is that execution stays
within $\mathcal{W}_{\text{ODD}}$ at $v \leq v_{\text{max}}$ and all
ODD-invalid sub-plans are blocked, constituting $A_F$ for the
Functional Layer.
\end{example}

A key open problem is establishing guarantees under distributional
shift: when $\hat{w}$ is drawn from a distribution that drifts from
the one under which $\mathcal{W}_{\text{ODD}}$ was certified, the
deterministic verdict may be unreliable, and robust conformal
methods~\cite{zhao2024robust} offer a promising direction.

\subsubsection{Functional Assurance Layer}\label{sec:functional-shield} 

The Functional Layer is the final assurance barrier, constraining
\emph{how} the system executes at control-loop frequency.
Once the User and Operational Layers have validated intent and
authorised the deployment context, the Functional Layer takes over
to enforce the physical, spatial, geometric, and temporal
constraints imposed by the deployment environment
dynamically---guaranteeing that the system remains within its safe
envelope even under dynamically changing conditions.

The layer draws on three complementary enforcement techniques.
Specification-based runtime monitoring tracks the executed
trajectory against safety specifications, producing robustness
margins $\rho_F(t)$ and verdicts that trigger intervention when
a violation is detected or anticipated.
Control barrier functions (CBFs) project any candidate control
command onto the admissible safe set via $\Pi_{\text{safe}}$,
correcting it minimally---if the command already satisfies all
constraints, $\Pi_{\text{safe}}$ leaves it unchanged. Recent work confirms that CBF-constrained quadratic programs can be
embedded inside LLM-driven planning loops without sacrificing task
completion, provided the dynamics model is sufficiently
accurate~\cite{khan2025safer,zhao2026sagellmsafegeneralizablellm}, and probabilistic
extensions now provide explicit bounds on safety failure probability
for learned barriers~\cite{urrea2026probabilistic,mestres2025probabilistic}. 
Simulation-based synthesis uses proxy setups such as world models
to perform pre-computations, for instance via safe reinforcement
learning, when the environment is large, unknown, or dynamic.
These three techniques are distributed across the quadruple
$(M_F, G_F, \pi_F, \Sigma_F)$: the monitor $M_F$ implements
specification-based monitoring; the guardrail $G_F$ applies CBF projection and simulation-based synthesis, with CBF
projection taking precedence when the two conflict at runtime; and when $G_F$ cannot
maintain safety---either because $\Pi_{\text{safe}}$ becomes
infeasible or $\rho_F(t) < 0$---the fallback $\pi_F$ is
activated, issuing $s_F \in \{\text{nominal}, \text{override},
\text{MRC}\}$ and triggering corrective action or safe-stop.

The safety contract $\Sigma_F = (A_F, \Gamma_F)$ closes the
contract chain.
Given validated plan $p_O$, fused constraints
$\mathrm{caps}_{\text{final}}$ and $\mathrm{zones}_{\text{final}}$,
candidate control $u_{\text{des}}(t)$, system state $x(t)$, and
system dynamics $(f, g)$, the layer produces a safe actuation
command $u_{\text{safe}}(t)$, a runtime robustness margin
$\rho_F(t)$, and a safety signal: 
$u_{\text{safe}},\, \rho_F,\, s_F
    \;=\; F(p_O,\, u_{\text{des}},\, x(t),\,
    \mathrm{caps}_{\text{final}},\, \mathrm{zones}_{\text{final}})$.
Assuming $A_F$---$w \in \mathcal{W}_{\text{ODD}}$ and $v \leq v_{\text{max}}$,
delivered by $\Gamma_O$, thereby closing the contract
chain---the layer guarantees $\Gamma_F$---no collision occurs,
all physical interactions satisfy $f \leq f_{\text{max}}$, and
any violation triggers corrective action or safe-stop.
Feedback flows bottom-up: $s_F$ is reported to the Operational
Layer, which in turn propagates $s_O$ to the User Layer, enabling
plan recomputation or user notification when execution cannot be
completed safely.

\begin{example}[Functional Layer]\label{ex:functional}
A resident enters the corridor en route to Room~12: $M_F$ detects
$d < d_{\min}$ and issues a hold.
Inside Room~12, a CBF constrains arm force to $f \leq f_{\max}$
throughout the handover, adapting to the resident's motion.
In Room~18, rearranged furniture blocks the delivery point; the
layer escalates to the Operational Layer with an updated obstacle
map.
Given $A_F$---namely $w \in \mathcal{W}_{\text{ODD}}$ and
$v \leq v_{\text{max}}$, delivered by $\Gamma_O$---the guarantee
$\Gamma_F$ is that no collision occurs and all physical
interactions satisfy $f \leq f_{\max}$, with violations triggering
corrective action or safe-stop.
\end{example}

Key open problems include guaranteeing real-time safety under
non-stationary dynamics and adapting safety envelopes
mid-execution when updated caps or zones arrive from upper layers.

\subsection{End-to-End Safety Guarantee}\label{sec:illustration}

Since $\Gamma_U \Rightarrow A_O$ and $\Gamma_O \Rightarrow A_F$, the three
contracts compose sequentially. Let $F_i = \neg\Gamma_i$ be the failure event of
layer~$i$ and $p_i = P(\Gamma_i)$ its satisfaction probability. Note that $G_U$
produces a deterministic verdict for every input, but determinism does not imply
correctness: over the distribution of possible plans, $G_U$ may still approve
non-compliant plans, so $p_U < 1$ in general. Four bounds of increasing tightness
characterise $P(\text{system safe}) = P(\Gamma_U \cap \Gamma_O \cap \Gamma_F)$:
\begin{align}
  P(\text{safe}) &\geq \max\!\bigl(0,\;p_U + p_O + p_F - 2\bigr)
      \tag{B1}\label{eq:B1-apx}\\
  P(\text{safe}) &\geq p_U + p_O + p_F - 2
                   + \textstyle\sum_{i<j} P(F_i \cap F_j)
      \tag{B2}\label{eq:B2-apx}\\
  P(\text{safe}) &= 1 - P(F_U \cup F_O \cup F_F)
      \tag{B3}\label{eq:B3-apx}\\
  P(\text{safe}) &= p_U \cdot p_{O|U} \cdot p_{F|OU}
      \tag{B4}\label{eq:B4-apx}
\end{align}

\eqref{eq:B1-apx} is the Fr\'{e}chet--Bonferroni bound~\cite{frechet1935,kwerel1975},
derivable from marginals alone and requiring no independence assumption;
it is the tightest bound so derivable~\cite{kwerel1975}. \eqref{eq:B2-apx} adds pairwise co-failure
probabilities $P(F_i \cap F_j)$, estimable from paired execution traces; it is
strictly tighter than~\eqref{eq:B1-apx} whenever layers share failure
modes~\cite{kwerel1975}. \eqref{eq:B3-apx} is exact via inclusion-exclusion; under mutual
independence it reduces to $p_U \cdot p_O \cdot p_F$. \eqref{eq:B4-apx} is the sequential
chain-rule decomposition, where $p_{O|U} = P(\Gamma_O \mid \Gamma_U)$ and
$p_{F|OU} = P(\Gamma_F \mid \Gamma_O,\Gamma_U)$; it is exact without any
independence assumption. When upstream filtering is effective---i.e., when the User Layer removes plans that would cause downstream failures---$p_{O|U} \geq p_O$ and $p_{F|OU} \geq p_F$, giving
$p_U \cdot p_{O|U} \cdot p_{F|OU} \geq p_U \cdot p_O \cdot p_F$. Whether this holds in a given deployment is an empirical property of the interaction between layers, not a structural guarantee; what the architecture guarantees is that the conditional factors are well-defined and independently certifiable. Any weakening of
any single layer degrades all four bounds, motivating independent certification
of all three layers.

\paragraph{Estimation caveat.}
Bound~\eqref{eq:B4-apx} is the natural certification target, but instantiating any of
\eqref{eq:B1-apx}--\eqref{eq:B4-apx} requires estimating quantities---$p_U$, $p_O$, $p_F$, co-failure
rates, and the conditionals $p_{O|U}$, $p_{F|OU}$---that standard PAC
theory~\cite{valiant1984} cannot supply directly, because LLM agent traces
violate the i.i.d.\ assumption in two ways: each step conditions on prior context
(non-stationarity), and all three layers share the same model backbone, coupling
their failure events. Several frameworks offer partial remedies.
Martingale-based bounds---used by Lotfi et al.~\cite{lotfi2024tokens} to derive
token-level generalisation bounds for LLMs---tolerate within-sequence dependence
without requiring stationarity.
Stability-based bounds for mixing processes~\cite{mohri2009,mohri2010} give
algorithm-specific guarantees for $\beta$- and $\phi$-mixing sequences,
subsuming i.i.d.\ as a special case, at the cost of characterising the mixing
rate of LLM traces.
Non-exchangeable conformal prediction~\cite{barber2023} provides distribution-free
coverage with degradation bounds quantified as a function of total-variation
distance from exchangeability.
Anytime-valid inference via e-processes~\cite{ramdas2023} yields confidence
sequences valid at all stopping times, fitting naturally with the incremental
re-certification of Section~\ref{sec:adaptivity}.
Most directly, Davidov et al.~\cite{davidov2026} reframe LLM safety probability
estimation as a survival analysis problem, constructing calibrated lower bounds on
time-to-unsafe-sampling via conformal prediction with finite-sample coverage
guarantees and no distributional assumptions.
None of these simultaneously handles non-stationarity, heterogeneous layer composition,
and correlated backbone failures; closing that gap is the open problem of
Section~\ref{sec:limitations}. For \eqref{eq:B1-apx} alone, martingale or mixing-process bounds are the most plausible near-term route, as they require only marginal estimates $p_U, p_O, p_F$
under within-sequence dependence. Until the full gap is closed, \eqref{eq:B1-apx}--\eqref{eq:B4-apx} provide a
principled certification target rather than deployable certificates.

Together, Examples~\ref{ex:user}-\ref{ex:functional} confirm that the composition operates as a live 
bidirectional assurance loop rather than a one-pass pipeline: forward 
contract propagation ($\Gamma_U \Rightarrow A_O \Rightarrow \Gamma_O 
\Rightarrow A_F \Rightarrow \Gamma_F$) certifies each stage before 
execution, while bottom-up safety signals ($s_F \rightarrow s_O 
\rightarrow s_U$) enable plan recomputation when execution cannot be 
completed safely within the certified envelope.

\section{Why Three Layers Are Structurally Necessary}
\label{app:impossibility}

This appendix formalises the structural argument of Section~\ref{sec:background-main}.
We do \emph{not} claim that single-layer certification is logically
impossible under all conceivable definitions.
We claim something more specific and more useful: any certification
architecture that (i)~respects the causal timeline of agent execution
and (ii)~satisfies the three desiderata below must implement the staged
information structure of the three-layer design---either explicitly or
implicitly.
The three layers are therefore not one solution among many but the
\emph{shape that any adequate solution must take}.

\paragraph{Formal setup (retained from main text).}
An LLM agent execution is a triple $e=(p,w,x)$ where $p\in\mathcal{P}$
is the LLM-generated plan, $w\in\mathcal{W}$ is the realised world state
observed via sensor data $\xi$ at execution time, and
$x=(x(t),u(t))_{t\ge 0}$ is the state--control trajectory during
closed-loop execution.
Define three sub-$\sigma$-algebras of the underlying probability space
$(\Omega,\mathcal{F},\mathbb{P})$:
\[
  \mathcal{I}_U = \sigma(y,\mathcal{P}_{\text{user}},\text{role}),
  \qquad
  \mathcal{I}_O = \sigma(\mathcal{I}_U,\xi,\hat{w},\mathcal{W}_{\text{ODD}}),
  \qquad
  \mathcal{I}_F = \sigma(\mathcal{I}_O,x(t),u(t)),
\]
and three Boolean safety predicates:
$\Phi_U(e)$ (semantic safety, $\mathcal{I}_U$-measurable),
$\Phi_O(e)$ (operational safety, $\mathcal{I}_O$-measurable but not
$\mathcal{I}_U$-measurable in general), and
$\Phi_F(e)$ (dynamical safety, $\mathcal{I}_F$-measurable but not
$\mathcal{I}_O$-measurable in general).

\subsection{Information-Set Ordering}

\begin{remark}[Strict inclusion of information sets]
\label{rem:strict-inclusion}
The inclusions $\mathcal{I}_U \not\subseteq \mathcal{I}_O \not\subseteq
\mathcal{I}_F$ are \emph{strict} under any operationally non-degenerate
deployment distribution $\mathcal{D}$
(Definition~\ref{def:nondeg} below).
\begin{itemize}
  \item \emph{($\mathcal{I}_U \not\subseteq \mathcal{I}_O$).}
    Two executions $(p,w,x)$ and $(p,w',x')$ sharing the same plan,
    intent, and role metadata are $\mathcal{I}_U$-indistinguishable.
    Setting $w\in\mathcal{W}_{\text{ODD}}$ and
    $w'\notin\mathcal{W}_{\text{ODD}}$ (possible whenever
    $\alpha:=\mathbb{P}_\mathcal{D}(w\notin\mathcal{W}_{\text{ODD}})>0$)
    yields $\Phi_O(e)=1$ and $\Phi_O(e')=0$, so $\Phi_O$ is not
    $\mathcal{I}_U$-measurable.
  \item \emph{($\mathcal{I}_O \not\subseteq \mathcal{I}_F$).}
    Two executions $(p,w,x)$ and $(p,w,x')$ sharing plan and world state
    are $\mathcal{I}_O$-indistinguishable.
    LLM non-determinism and stochastic actuation can yield trajectories
    $x$ satisfying all dynamical constraints while $x'$ incurs a
    collision at some $t>0$ (possible whenever
    $\beta:=\mathbb{P}_\mathcal{D}(\Phi_F=0\mid w\in
    \mathcal{W}_{\text{ODD}})>0$), so $\Phi_F$ is not
    $\mathcal{I}_O$-measurable.
\end{itemize}
This ordering is a direct consequence of the physical execution timeline
and holds regardless of the computational architecture employed.
\end{remark}

\begin{definition}[Operationally non-degenerate distribution]
\label{def:nondeg}
A distribution $\mathcal{D}$ over executions $e=(p,w,x)$ is
\emph{operationally non-degenerate} if $\alpha>0$ and $\beta>0$
as defined above.
Both conditions are mild and hold in any realistic deployment.
\end{definition}

\subsection{Three Causally-Grounded Desiderata}

Three
desiderata follow directly from the physical and institutional
structure of agent deployment rather than from a design choice about
contract semantics.
Their grounding in the causal structure of action---the distinction
between what information is available \emph{before} an action versus
\emph{after}---is in the tradition of structural-equation models of
causality \citep{halpern2016actual}.

\begin{desideratum}[Preventive semantic gate]\label{des:D1}
Semantic authorisation ($\Phi_U$) must be certified
\emph{before execution begins}, using only information available
prior to any world observation.
A scheme that permits execution of a semantically unauthorised plan
and attempts correction afterward is an \emph{incident-response}
architecture, not a safety architecture.
This is not a design preference: actions in the physical world are
generically irreversible on the timescale of plan execution,
so post-hoc correction cannot restore the pre-execution state.
\end{desideratum}

\begin{desideratum}[Causal observability of operational safety]\label{des:D2}
Operational safety ($\Phi_O$) depends on the realised world state $w$,
which is physically unavailable until sensor data $\xi$ is collected at
execution time.
No pre-execution computation can substitute for this observation without
either (a) assuming the world state in advance---reintroducing
distributional-shift failure modes---or (b) adopting a worst-case bound
that blocks all non-trivially constrained execution.
Neither substitute constitutes a valid operational safety certificate.
\end{desideratum}

\begin{desideratum}[Trajectory dependence of dynamical safety]\label{des:D3}
Dynamical safety ($\Phi_F$) depends on the state--control trajectory
$(x(t),u(t))_{t\ge 0}$, which does not exist until the control loop is
running.
Pre-execution or pre-actuation certification of $\Phi_F$ is either a
simulation---inheriting all limitations of the world model---or a
worst-case envelope too conservative for practical deployment.
Neither substitute constitutes runtime enforcement.
\end{desideratum}

\medskip
These desiderata are not definitional impositions on the A/G framework.
D1 reflects the irreversibility of physical action; D2 and D3 reflect the
causal precedence of observation over inference: a property that depends
on data not yet available cannot be certified before that data arrives
\citep{halpern2016actual}.
Together, D1--D3 establish the strict temporal ordering
$\tau_U < \tau_O < \tau_F$ of certification stages, where $\tau_i$
denotes the earliest time at which $\Phi_i$ can be certified.
This ordering coincides exactly with the strict inclusion
$\mathcal{I}_U \not\subseteq \mathcal{I}_O \not\subseteq \mathcal{I}_F$
of Remark~\ref{rem:strict-inclusion}.

\subsection{The Collapse Argument}

\begin{proposition}[Collapse argument]
\label{prop:collapse}
Any architecture with fewer than three independently certified stages,
each satisfying one of \textrm{D1}--\textrm{D3}, must either
(i)~violate at least one desideratum, or
(ii)~implicitly reconstruct the three-stage structure within a nominally
single component.
\end{proposition}

\begin{proof}
There are three cases, corresponding to the three distinct ways of
collapsing two of the three stages.

\medskip
\noindent\textbf{Case A: Collapse $\mathcal{I}_U$ and $\mathcal{I}_O$
(single pre-execution stage).}
A certifier operating only on pre-execution information must decide on
$\Phi_O$ before sensor data $\xi$ is available.
By D2, no pre-execution computation yields a valid $\Phi_O$ certificate
without either assuming $w$ or reverting to a worst-case bound.
Alternatively, if the certifier defers the $\Phi_U$ decision until
after world-state observation is incorporated, D1 is violated: a
semantically unauthorised plan may already have initiated actions by the
time the joint check completes.
In either sub-case, at least one desideratum is violated.
The specific failure mode is the one identified in Section~\ref{sec:background-main} for
existing operational-layer approaches: ODD membership is assumed rather
than verified, or semantic authorisation is not a true pre-execution gate.

\medskip
\noindent\textbf{Case B: Collapse $\mathcal{I}_O$ and $\mathcal{I}_F$
(single execution-time stage).}
A certifier that simultaneously decides ODD membership (a deterministic
binary verdict required for downstream composability) and enforces
dynamical safety at control-loop frequency entangles two logically and
temporally distinct decisions.
The ODD verdict---which must be stable and auditable as it constitutes
assumption $A_F$ for the dynamical controller---is then contingent on
real-time trajectory data that changes at every control step.
The result is that $A_F$ is no longer a well-formed, stable assumption:
the dynamical controller cannot rely on a fixed operational certificate,
and the contract chain $\Gamma_O \Rightarrow A_F$ is broken.
Additionally, collapsing these stages removes the autonomy-envelope
governance function of the Operational Layer entirely, since autonomy
decisions require world-state estimation without yet having trajectory
data.

\medskip
\noindent\textbf{Case C: Collapse $\mathcal{I}_U$ and $\mathcal{I}_F$,
leaving $\mathcal{I}_O$ separate.}
A certifier must simultaneously check $\Phi_U$---which requires only
$\mathcal{I}_U$ and must be completed before any world observation
(D1)---and $\Phi_F$---which requires $\mathcal{I}_F$ and is only
available during control-loop execution (D3).
These two tasks operate at strictly different points in time and over
strictly disjoint information increments.
Any implementation that performs both necessarily executes them
sequentially: first the $\mathcal{I}_U$-based check, then the
$\mathcal{I}_F$-based enforcement.
This sequential computation implicitly defines a boundary between
$\mathcal{I}_U$ and $\mathcal{I}_F$ within the nominally single
component---reconstructing the two-stage structure rather than
eliminating it, while leaving $\mathcal{I}_O$-based certification
unaddressed.

\medskip
\noindent\textbf{Combining cases.}
In Case~A, $\Phi_O$ or $\Phi_U$ is uncertified; in Case~B, the
contract chain is broken at $\Gamma_O \Rightarrow A_F$; in Case~C,
$\Phi_F$ or $\Phi_U$ is deferred, or the two-stage structure is
implicitly reconstructed.
All three cases are exhaustive over the possible two-stage designs.
None achieves simultaneous satisfaction of D1, D2, and D3 without
either violating one or reconstructing the three-stage structure.
\end{proof}

\subsection{Necessity Corollary and the Reconstruction Remark}

\begin{corollary}[Three stages are necessary]
\label{cor:three-necessary}
The three-layer architecture of Section~3 is the \emph{minimum}
number of independently certified stages that can simultaneously satisfy
\textrm{D1}, \textrm{D2}, and \textrm{D3} under any operationally
non-degenerate $\mathcal{D}$.
Any two-stage design necessarily violates at least one desideratum or
implicitly reconstructs the three-stage structure.
Additional stages may refine the architecture---for instance by
splitting $\mathcal{I}_U$ into separate cognitive, regulatory, and
ethical sub-layers---but cannot reduce the minimum.
\end{corollary}

\begin{remark}[On implicit reconstruction]
\label{rem:reconstruction}
The Collapse Argument shows that a sufficiently powerful
single certifier that attempts to handle all three safety dimensions
must perform staged computation, implicitly partitioning its input into
the three information sets at different points in time.
A single component that does this is \emph{a three-layer architecture
with internal boundaries hidden from external inspection}---harder to
independently certify, harder to audit, and with no compositional safety
bound available between stages, since the intermediate guarantees are not
made explicit.
Making the boundaries explicit is therefore not a taxonomic preference
but a \emph{certification requirement}: without explicit boundaries,
the layer-level probabilities $p_U$, $p_O$, $p_F$ and the
conditionals $p_{O|U}$, $p_{F|OU}$ entering bounds (B1)--(B4) cannot be
independently estimated or verified.
\end{remark}

\begin{remark}[On the strength of the desiderata]
One might ask whether \textrm{D1}--\textrm{D3} could be weakened: specifically,
whether a sufficiently powerful single certifier on $\mathcal{I}_O$ could perform
two-phase computation, first checking $\Phi_U$ using only its
$\mathcal{I}_U$-measurable component, then checking $\Phi_O$, thereby handling
two safety dimensions without explicitly separating the layers.
The Collapse Argument shows this does not succeed: such a two-phase certifier
\emph{reconstructs} rather than avoids the layered structure
(Case~C of Proposition~1 and Remark~2), and leaves $\Phi_F$ uncertified in any
case.
The desiderata \textrm{D1}--\textrm{D3} are grounded in physical causality---the
irreversibility of action and the causal precedence of observation over
inference~\citep{halpern2016actual}---rather than in a design choice about A/G
contract semantics, and cannot be dissolved by redefining the certification
boundary.
\end{remark}

\subsection{Bound Degradation Under Collapse}
\label{sec:bound-degradation}

The Collapse Argument (Proposition~\ref{prop:collapse}) establishes
that any two-stage design must violate at least one of the desiderata
D1--D3, or implicitly reconstruct the three-stage structure.
The following proposition establishes the quantitative counterpart:
the conditional factors of (B4) are not merely degraded but
structurally unverifiable under any genuine two-stage collapse.
Together, Propositions~\ref{prop:collapse}
and~\ref{prop:bound-degradation} close the logical gap between the
necessity argument of Section~\ref{sec:background-main} and the
quantitative framework of Section~\ref{sec:bounds}: the
three-layer architecture is not only the minimal structure satisfying
D1--D3, but the unique structure under which (B4) can be
independently verified as a product of auditable layer-level quantities.

\begin{proposition}[Certification Degradation Under Collapse]
\label{prop:bound-degradation}
Let $p_U \cdot p_{O|U} \cdot p_{F|OU}$ denote the chain-rule
decomposition (B4) of the three-layer architecture.
Under any two-stage collapse (Cases~A, B, C of
Proposition~\ref{prop:collapse}), at least one factor of (B4) is
either \emph{undefined} as a well-formed conditioning event or
\emph{unverifiable} as an independently certified probability.
Consequently, no genuine two-stage design yields (B4) as a product
of independently estimable, auditable layer-level quantities.
\end{proposition}

\begin{proof}
We treat each collapse case in turn.

\paragraph{Case~A: Collapse of $\mathcal{I}_U$ and $\mathcal{I}_O$
(single pre-execution stage).}
The combined certifier must issue the ODD-membership verdict
($\Phi_O$) using only $\mathcal{I}_U$-measurable information.
By Remark~\ref{rem:strict-inclusion}, $\Phi_O$ is not
$\mathcal{I}_U$-measurable: two executions $(p,w,x)$ and
$(p,w',x')$ sharing the same plan are
$\mathcal{I}_U$-indistinguishable yet can have opposite
ODD-membership outcomes ($w \in \mathcal{W}_{\mathrm{ODD}}$,
$w' \notin \mathcal{W}_{\mathrm{ODD}}$) whenever
$\alpha := \mathcal{D}(w \notin \mathcal{W}_{\mathrm{ODD}}) > 0$.
A certifier confined to $\mathcal{I}_U$ therefore cannot label
execution traces with $\Gamma_O$ outcomes---which requires
observing the world state $w$---and hence cannot estimate
$p_{O|U} = \Pr(\Gamma_O \mid \Gamma_U)$ from pre-execution
information.
Any estimate must take one of three forms:
\emph{(a1)} assume $w \in \mathcal{W}_{\mathrm{ODD}}$ throughout,
  yielding an unverified surrogate rather than a certified conditional;
\emph{(a2)} adopt a worst-case bound, blocking most executions and
  rendering the certificate vacuously conservative; or
\emph{(a3)} integrate over a prior $\pi(w)$, yielding
  the marginal $p_O$ rather than the conditional $p_{O|U}$, since the
  conditioning event $\Gamma_U$ carries no information about $w$
  when the certifier cannot observe $w$.
In all three sub-cases, $p_{O|U}$ is not independently estimable as a
conditional quantity grounded in observed world states.
The factor is unverifiable.

\paragraph{Case~B: Collapse of $\mathcal{I}_O$ and $\mathcal{I}_F$
(single execution-time stage).}
The factor $p_{F|OU} = \Pr(\Gamma_F \mid \Gamma_O, \Gamma_U)$
requires $\Gamma_O$ to be a \emph{stable, pre-actuation event}:
the ODD verdict must be issued and fixed before the Functional
Layer begins executing, so that it constitutes a well-formed
conditioning event for $\Gamma_F$.
In a genuine collapse of $\mathcal{I}_O$ and $\mathcal{I}_F$,
however, there is no pre-actuation point at which an ODD verdict
can be issued independently of trajectory data: any such verdict
would require $\mathcal{I}_F$-information (the trajectory $x(t)$)
that does not yet exist, or it would require separating the
$\mathcal{I}_O$-based sub-process from the $\mathcal{I}_F$-based
sub-process within the nominally single component.
If the verdict is issued before $x(t)$ exists by separating the
two sub-processes, the architecture has reconstructed the explicit
$\mathcal{I}_O$/$\mathcal{I}_F$ boundary (Remark~\ref{rem:reconstruction})
and is not a genuine two-stage design.
If the verdict is not separated, $\Gamma_O$ is determined jointly
with $\Gamma_F$ as the trajectory unfolds, making the conditioning
event in $p_{F|OU}$ trajectory-indexed and therefore not a fixed,
pre-actuation event against which $\Gamma_F$ can be independently
certified.
The factor $p_{F|OU}$ is undefined as a static certification
quantity, and (B4) cannot be written as a well-formed product.

\paragraph{Case~C: Collapse of $\mathcal{I}_U$ and $\mathcal{I}_F$,
leaving $\mathcal{I}_O$ separate.}
Since the Operational Layer remains independent, $\Gamma_O$ is
well-defined and $p_{O|U}$ is estimable.
The issue concerns the joint certification of $p_U$ and $p_{F|OU}$.
By Proposition~\ref{prop:collapse} Case~C, any certifier that
simultaneously handles $\Phi_U$ (pre-execution,
$\mathcal{I}_U$-measurable) and $\Phi_F$ (control-loop,
$\mathcal{I}_F$-measurable) must execute these checks sequentially,
implicitly defining an internal boundary.
Two sub-cases arise.
\emph{Sub-case~(a): the boundary is hidden.}
Both $\Gamma_U$ and $\Gamma_F$ are outputs of the same component;
their failure events $F_U$ and $F_F$ are generated by the same
internal process.
Independent certification of $p_U$ and $p_{F|OU}$---as required
by the modular guarantee structure of Appendix~\ref{sec:illustration},
which composes ``locally certifiable bounds into a system-level
guarantee''---requires each probability to be estimated from an
independently auditable sub-system.
With a hidden boundary, neither can be audited without auditing
the other.
In particular, the partial-decoupling strategy of Section~4---deriving
$\Gamma_F$ from a CBF certificate whose dynamics model does not
share parameters with the LLM reasoning engine that produces
$\Gamma_U$~\cite{pandya2025}---is unavailable: both verdicts
originate in the same component, so the cross-term
$\Pr(F_U \cap F_F)$ is uncontrolled and cannot be bounded by
partial independence.
The factors $p_U$ and $p_{F|OU}$ in (B4) are not independently
verifiable, and (B4) cannot be certified as a product of
independently audited quantities.
\emph{Sub-case~(b): the boundary is made explicit.}
If the component separately certifies its
$\mathcal{I}_U$-based sub-process and its $\mathcal{I}_F$-based
sub-process, with each independently auditable, then---combined
with the separate Operational Layer---the architecture has
reconstructed the full three-stage structure
(Remark~\ref{rem:reconstruction}).
This is not a genuine two-stage design.

\paragraph{Combining cases.}
In Case~A, $p_{O|U}$ is not independently estimable.
In Case~B, the factor $p_{F|OU}$ is undefined as a well-formed
conditioning event.
In Case~C, either $p_U$ and $p_{F|OU}$ are not independently
certifiable (sub-case~a), or the three-stage structure has been
reconstructed (sub-case~b).
Cases~A, B, C are exhaustive over all two-stage designs
(Proposition~\ref{prop:collapse}).
In no case does a genuine two-stage architecture yield (B4) as a
product of independently estimable and auditable layer-level
quantities.
\end{proof}

\begin{remark}[Structural and certification necessity are equivalent]
\label{rem:quantitative-necessity}
Proposition~\ref{prop:collapse} establishes that three independently
certified stages are the \emph{minimum} satisfying D1--D3
(structural necessity).
Proposition~\ref{prop:bound-degradation} establishes that three
independently certified stages are \emph{necessary and sufficient}
for (B4) to be a well-formed, independently verifiable
certification target (certification necessity).
These are not two separate arguments that happen to agree; both
follow from the same temporal ordering
$\tau_U < \tau_O < \tau_F$ of certifiable information.
\end{remark}

\section{Numerical Instantiation of the Safety Bounds}
\label{app:numerical}

This appendix instantiates bounds (B1)--(B4) on the caregiver-robot
running example using illustrative but empirically grounded estimates.
The purpose is not to certify a real deployment but to show that the
bounds produce meaningful numbers, that the conditional structure of (B4) can yield
quantifiable gains when upstream filtering is effective, and that the architecture's sensitivity to
individual layer weakening is concrete rather than asymptotic.

\subsection*{C.1\quad Layer-Level Probability Estimates}

We assign marginal satisfaction probabilities anchored to published
empirical results for comparable deployed components.

\paragraph{User Layer ($p_U = 0.95$).}
ShieldAgent achieves 90.1\% rule recall on ShieldAgent-Bench across
six web environments and seven risk categories, outperforming prior
methods by 11.3\% on average~\cite{chen2025shieldagent}.
More directly relevant to the embodied deployment context of the
running example, AgentSpec eliminates all hazardous actions across
every unsafe task category in embodied agent experiments
(100\% hazard prevention rate) and prevents unsafe executions in
over 90\% of code-agent cases~\cite{wang2025agentspeccustomizableruntimeenforcement}.
We set $p_U = 0.95$, conservative relative to AgentSpec's
embodied result, to account for residual exposure to informal
intent, policy ambiguity, and adversarial
manipulation~\cite{li2025mind}.

\paragraph{Operational Layer ($p_O = 0.90$).}
Without a dedicated ODD monitor, frontier LLMs fail to detect
12.24\% of direct out-of-domain queries under non-adversarial
conditions, rising to 70.72\% under adversarial
transformations~\cite{lei2026offtopiceval}.
We set $p_O = 0.90$, assuming a purpose-built deterministic ODD
checker (augmented with robust conformal
methods~\cite{zhao2024robust}) rather than a raw LLM call,
consistent with the deterministic-verdict design of
Section~\ref{sec:framework}. The gap $1-p_O = 0.10$ represents
residual failure under non-adversarial ODD boundary conditions.

\paragraph{Functional Layer ($p_F = 0.92$).}
CBF-constrained quadratic programs embedded in LLM planning loops
reduce safety violations substantially while maintaining task
efficiency in hardware experiments involving heterogeneous robotic
agents~\cite{khan2025safer}. Probabilistic CBF extensions now
provide explicit bounds on safety-failure probability for learned
barriers~\cite{mestres2025probabilistic,urrea2026probabilistic}.
We set $p_F = 0.92$, reflecting near-collision-free operation
under moderate dynamic uncertainty; the 8\% gap captures LLM
non-determinism, stochastic actuation, and scenarios such as the
rearranged furniture in Room~18 requiring mid-execution replanning.

\subsection*{C.2\quad Conditional Estimates and Upstream Filtering}

The chain-rule decomposition (B4) requires estimating the conditional
probabilities $p_{O|U}$ and $p_{F|OU}$. These conditionals reflect
the performance of each layer on inputs that have passed upstream
certification. Whether these conditionals exceed the corresponding
marginals is an empirical property of the deployment, not a structural
guarantee (see footnote in Section~\ref{sec:bounds}).

In the running example, it is plausible that User Layer filtering
improves downstream performance: blocking rest-hour entries,
camera-routing plans, and plans omitting conditional water-delivery
logic removes 4--6\% of plans that would require non-trivial ODD
adjudication. We therefore consider the \emph{illustrative} estimates
$p_{O|U} = 0.94$ and $p_{F|OU} = 0.96$, representing a scenario in
which upstream filtering is beneficial. For comparison, we also
evaluate (B4) under the conservative assumption $p_{O|U} = p_O = 0.90$
and $p_{F|OU} = p_F = 0.92$ (no filtering benefit), which recovers
the na\"{i}ve product.

Estimating these conditionals precisely from non-i.i.d.\ traces
remains the open problem of Section~4, regardless of whether they
exceed the marginals.

\subsection*{C.3\quad Bound Evaluation}

Table~\ref{tab:bounds} evaluates (B1)--(B4) under
$(p_U, p_O, p_F) = (0.95, 0.90, 0.92)$ and
$(p_{O|U}, p_{F|OU}) = (0.94, 0.96)$.
For (B2) we use pairwise co-failure rate $\rho = 0.02$
for all $i < j$, reflecting moderate positive correlation from the
shared LLM backbone (Section~4); sensitivity to $\rho$ is
analysed in Section~C.5.

\begin{table}[h]
\centering
\caption{Safety bounds on the caregiver-robot example.
  Marginals: $(p_U, p_O, p_F) = (0.95, 0.90, 0.92)$.
  Conditionals: $p_{O|U} = 0.94$, $p_{F|OU} = 0.96$.
  Co-failure rate (B2): $\rho = 0.02$.
  Triple co-failure (B3): $\Pr(F_U\cap F_O\cap F_F) = 0.004$.
  All quantities illustrative; see text for empirical grounding.
  The na\"{i}ve product assumes independence \emph{and} no filtering
  benefit ($p_{O|U} = p_O$, $p_{F|OU} = p_F$). The gap between
  (B4) and the na\"{i}ve product is deployment-dependent; see Section~C.4.}
\label{tab:bounds}
\renewcommand{\arraystretch}{1.25}
\resizebox{\textwidth}{!}{
\begin{tabular}{clcc}
\toprule
Bound & \multicolumn{1}{c}{Formula / note}
      & Value & Inputs required \\
\midrule
(B1)
  & $\max(0,\; p_U + p_O + p_F - 2)$
  & $0.770$
  & Marginals only \\
(B1$^{+}$)
  & Generalised BF~\cite{salako2025frechet} + structural
    constraints$^{\S}$
  & ${\geq}\,0.770$
  & Marginals + layer-independence facts \\
(B2)
  & $(B1) + \textstyle\sum_{i<j}\Pr(F_i \cap F_j)$
  & $0.830$
  & Marginals + pairwise co-failure \\
(B3)
  & $1 - \Pr(F_U \cup F_O \cup F_F)$
  & $0.826^{\dagger}$
  & Full inclusion-exclusion \\
(B4)
  & $p_U \cdot p_{O|U} \cdot p_{F|OU}$
  & $\mathbf{0.857}$
  & Conditionals; exact, no independence \\
\midrule
\multicolumn{2}{l}{\textit{Naive product (independence assumed,
  no filtering):} $p_U \cdot p_O \cdot p_F$}
  & $0.787$
  & Marginals + independence \\
\bottomrule
\end{tabular}
}
\smallskip

\raggedright\small
$^{\dagger}$~(B3) uses
$\Pr(F_U\!\cap\!F_O\!\cap\!F_F)=0.004$.
Computed as $1-(f_U+f_O+f_F - 3\rho + 0.004)
= 1-(0.05+0.10+0.08-0.06+0.004) = 1-0.174 = 0.826$.
Sensitive to the triple co-failure rate, which requires full
trace data to estimate.\\[4pt]
$^{\S}$~(B1$^{+}$) improves on (B1) when structural constraints
such as bounds on the number of layers that can simultaneously
fail are available. The CBF certificate of the Functional Layer
is derived from a dynamics model that does not share parameters
with the LLM backbone~\cite{pandya2025}, making a bound on
$\Pr(F_U\cap F_F)$ practically estimable and yielding a
tightened floor above 0.770.
\end{table}

\subsection*{C.4\quad Three Observations}

\paragraph{(i) When upstream filtering is beneficial, the gain is
substantial and compositional.}
Under the illustrative conditional estimates, the architecture achieves
$p_U \cdot p_{O|U} \cdot p_{F|OU}
= 0.95 \times 0.94 \times 0.96 = 0.857$, compared with the
na\"{i}ve product $p_U \cdot p_O \cdot p_F
= 0.95 \times 0.90 \times 0.92 = 0.787$ under independence and
no filtering benefit. The 7-point gap in this scenario arises
from the conditional structure of (B4) and illustrates the
potential value of upstream certification. Whether this gap is
realised in a given deployment is an empirical question; the
architecture's contribution is that it makes the relevant
conditionals well-defined and independently estimable, enabling
this question to be answered.

\paragraph{(ii) (B1) is the near-term deployable floor; it can
be systematically tightened.}
(B1) $= 0.770$ requires only marginal estimates, obtainable from
execution traces via martingale or mixing-process bounds
(Section~4), and provides a valid certification floor while the
harder conditional estimates for (B4) are developed.
It is however substantially more conservative than
(B4) under the illustrative conditional estimates ($= 0.857$):
the 8.7-point gap represents the cost of using marginals-only
information in this scenario, and motivates investment in
conditional estimation methods.
Salako's generalised Boole--Fréchet dynamic-programming
framework~\cite{salako2025frechet} shows this floor can be
systematically raised by incorporating structural constraints.
The CBF certificate's independence from the LLM
backbone~\cite{pandya2025} provides exactly such a constraint
for the $F_U \cap F_F$ pair, making (B1$^+$) a practically
reachable improvement.

\paragraph{(iii) A single weak layer degrades all bounds
irreversibly; no compensating strengthening suffices.}
Suppose the Operational Layer degrades to $p_O = 0.75$ under
persistent sensor degradation. Under the illustrative conditional
estimates, the effect is as follows.
\begin{align*}
\text{(B1):}\quad &\max(0,\;0.95+0.75+0.92-2) = 0.62
   &&\text{(was 0.770; loss of 15 pp)}\\
\text{(B1) with }p_U\!=\!p_F\!=\!0.99:\quad
   &\max(0,\;0.99+0.75+0.99-2) = 0.73
   &&\text{(still 4 pp below baseline)}
\end{align*}
The (B4) picture is identical in character: with $p_{O|U}$
degrading proportionally to $0.94 \times (0.75/0.90) = 0.783$,
\[
  p_U \cdot p_{O|U} \cdot p_{F|OU}
  = 0.95 \times 0.783 \times 0.96 = 0.714
  \quad\text{(was 0.857; loss of 14 pp).}
\]
Strengthening $p_U$ or $p_{F|OU}$ to 0.99 does not recover
this loss because the deficit enters multiplicatively through
$p_{O|U}$. This asymmetry motivates independent certification
of all three layers rather than investment concentrated in one.

\subsection*{C.5\quad Sensitivity to Co-failure Rate and a Caution
on (B2)}

Table~\ref{tab:sensitivity} reports (B2) as $\rho$ varies.
(B2) $= 0.770 + 3\rho$; it remains a valid lower bound for any
$\rho \leq 0.077$ (the point at which B2 would reach~1).

\begin{table}[h]
\centering
\caption{Sensitivity of (B2) to pairwise co-failure rate~$\rho$.
  (B1)$\,=\,0.770$ throughout.}
\label{tab:sensitivity}
\begin{tabular}{ccc}
\toprule
$\rho$ & (B2) & Gap over (B1) \\
\midrule
$0.005$ & $0.785$ & $+0.015$ \\
$0.010$ & $0.800$ & $+0.030$ \\
$0.020$ & $0.830$ & $+0.060$ \\
$0.030$ & $0.860$ & $+0.090$ \\
\bottomrule
\end{tabular}
\end{table}

A critical caution: higher $\rho$ makes (B2) look tighter while
true system risk simultaneously \emph{increases}.
A co-failure rate of $\rho = 0.030$ means all three layers share
substantial failure events---the backbone can bring all three down
at once---yet (B2) rises to 0.860, near (B4)'s 0.857 under
the illustrative conditional estimates.
This is precisely the danger identified in Section~4: under positive
correlation, (B1) and (B2) are valid but potentially misleading.
Until the conditional estimation gap is closed, (B4)
should be treated as the certification target while (B1) and (B2)
serve as conservative floors, not deployable certificates.

\newpage
\section*{NeurIPS Paper Checklist}

The checklist is designed to encourage best practices for responsible machine learning research, addressing issues of reproducibility, transparency, research ethics, and societal impact. Do not remove the checklist: {\bf The papers not including the checklist will be desk rejected.} The checklist should follow the references and follow the (optional) supplemental material.  The checklist does NOT count towards the page
limit. 

Please read the checklist guidelines carefully for information on how to answer these questions. For each question in the checklist:
\begin{itemize}
    \item You should answer \answerYes{}, \answerNo{}, or \answerNA{}.
    \item \answerNA{} means either that the question is Not Applicable for that particular paper or the relevant information is Not Available.
    \item Please provide a short (1--2 sentence) justification right after your answer (even for \answerNA). 
\end{itemize}

{\bf The checklist answers are an integral part of your paper submission.} They are visible to the reviewers, area chairs, senior area chairs, and ethics reviewers. You will also be asked to include it (after eventual revisions) with the final version of your paper, and its final version will be published with the paper.

The reviewers of your paper will be asked to use the checklist as one of the factors in their evaluation. While \answerYes{} is generally preferable to \answerNo{}, it is perfectly acceptable to answer \answerNo{} provided a proper justification is given (e.g., error bars are not reported because it would be too computationally expensive'' or ``we were unable to find the license for the dataset we used''). In general, answering \answerNo{} or \answerNA{} is not grounds for rejection. While the questions are phrased in a binary way, we acknowledge that the true answer is often more nuanced, so please just use your best judgment and write a justification to elaborate. All supporting evidence can appear either in the main paper or the supplemental material, provided in appendix. If you answer \answerYes{} to a question, in the justification please point to the section(s) where related material for the question can be found.

IMPORTANT, please:
\begin{itemize}
    \item {\bf Delete this instruction block, but keep the section heading ``NeurIPS Paper Checklist"},
    \item  {\bf Keep the checklist subsection headings, questions/answers and guidelines below.}
    \item {\bf Do not modify the questions and only use the provided macros for your answers}.
\end{itemize}


\begin{enumerate}

\item {\bf Claims}
    \item[] Question: Do the main claims made in the abstract and introduction accurately reflect the paper's contributions and scope?
    \item[] Answer: \answerYes{} 
    \item[] Justification: The abstract and introduction clearly state the 
central claim---that single-layer enforcement is structurally insufficient 
for LLM agent safety---and accurately scope it as a position paper 
presenting a formal structural argument (Appendix~B), a sketched 
architecture with probabilistic safety bounds (Section~3), and four 
open problems (Section~4) rather than empirical results. Aspirational 
goals such as a deployable standard are clearly flagged as future work.
    \item[] Guidelines:
    \begin{itemize}
        \item The answer \answerNA{} means that the abstract and introduction do not include the claims made in the paper.
        \item The abstract and/or introduction should clearly state the claims made, including the contributions made in the paper and important assumptions and limitations. A \answerNo{} or \answerNA{} answer to this question will not be perceived well by the reviewers. 
        \item The claims made should match theoretical and experimental results, and reflect how much the results can be expected to generalize to other settings. 
        \item It is fine to include aspirational goals as motivation as long as it is clear that these goals are not attained by the paper. 
    \end{itemize}

\item {\bf Limitations}
    \item[] Question: Does the paper discuss the limitations of the work performed by the authors?
    \item[] Answer: \answerYes{} 
    \item[] Justification: Section~4 explicitly discusses all major 
limitations: the framework is scoped to single-agent systems, treats 
the LLM as a fixed black box requiring re-estimation of probability 
bounds upon model change, and identifies three open problems---bound estimation from non-i.i.d.\ traces, graceful degradation of contracts under deployment drift,  and
extension to multi-agent settings---that stand between the architecture and a deployable 
standard. The two structural assumptions bounding applicability are 
stated explicitly at the opening of Section~4. 
    \item[] Guidelines:
    \begin{itemize}
        \item The answer \answerNA{} means that the paper has no limitation while the answer \answerNo{} means that the paper has limitations, but those are not discussed in the paper. 
        \item The authors are encouraged to create a separate ``Limitations'' section in their paper.
        \item The paper should point out any strong assumptions and how robust the results are to violations of these assumptions (e.g., independence assumptions, noiseless settings, model well-specification, asymptotic approximations only holding locally). The authors should reflect on how these assumptions might be violated in practice and what the implications would be.
        \item The authors should reflect on the scope of the claims made, e.g., if the approach was only tested on a few datasets or with a few runs. In general, empirical results often depend on implicit assumptions, which should be articulated.
        \item The authors should reflect on the factors that influence the performance of the approach. For example, a facial recognition algorithm may perform poorly when image resolution is low or images are taken in low lighting. Or a speech-to-text system might not be used reliably to provide closed captions for online lectures because it fails to handle technical jargon.
        \item The authors should discuss the computational efficiency of the proposed algorithms and how they scale with dataset size.
        \item If applicable, the authors should discuss possible limitations of their approach to address problems of privacy and fairness.
        \item While the authors might fear that complete honesty about limitations might be used by reviewers as grounds for rejection, a worse outcome might be that reviewers discover limitations that aren't acknowledged in the paper. The authors should use their best judgment and recognize that individual actions in favor of transparency play an important role in developing norms that preserve the integrity of the community. Reviewers will be specifically instructed to not penalize honesty concerning limitations.
    \end{itemize}

\item {\bf Theory assumptions and proofs}
    \item[] Question: For each theoretical result, does the paper provide the full set of assumptions and a complete (and correct) proof?
    \item[] Answer: \answerYes{} 
    \item[] Justification: All theoretical results state their assumptions 
explicitly. The structural necessity argument is formalised in 
Appendix~B with a complete proof of Proposition~1 (Collapse Argument) 
covering all three cases, Corollary~1, and Remarks~2--4. The 
probabilistic bounds (B1)--(B4) are stated with full assumptions in 
Section~3.2 and elaborated in Appendix~A.2, including the estimation 
caveat and the conditions under which each bound is tighter than the 
others. 
    \item[] Guidelines:
    \begin{itemize}
        \item The answer \answerNA{} means that the paper does not include theoretical results. 
        \item All the theorems, formulas, and proofs in the paper should be numbered and cross-referenced.
        \item All assumptions should be clearly stated or referenced in the statement of any theorems.
        \item The proofs can either appear in the main paper or the supplemental material, but if they appear in the supplemental material, the authors are encouraged to provide a short proof sketch to provide intuition. 
        \item Inversely, any informal proof provided in the core of the paper should be complemented by formal proofs provided in appendix or supplemental material.
        \item Theorems and Lemmas that the proof relies upon should be properly referenced. 
    \end{itemize}

    \item {\bf Experimental result reproducibility}
    \item[] Question: Does the paper fully disclose all the information needed to reproduce the main experimental results of the paper to the extent that it affects the main claims and/or conclusions of the paper (regardless of whether the code and data are provided or not)?
    \item[] Answer: \answerNA{} 
    \item[] Justification: This is a position paper with no experiments. 
All empirical figures cited (AgentSafetyBench, 
Agent Security Bench) are drawn from published third-party benchmarks 
with full citations, and are used only to motivate the structural 
argument rather than as primary evidence for the claims. 
    \item[] Guidelines:
    \begin{itemize}
        \item The answer \answerNA{} means that the paper does not include experiments.
        \item If the paper includes experiments, a \answerNo{} answer to this question will not be perceived well by the reviewers: Making the paper reproducible is important, regardless of whether the code and data are provided or not.
        \item If the contribution is a dataset and\slash or model, the authors should describe the steps taken to make their results reproducible or verifiable. 
        \item Depending on the contribution, reproducibility can be accomplished in various ways. For example, if the contribution is a novel architecture, describing the architecture fully might suffice, or if the contribution is a specific model and empirical evaluation, it may be necessary to either make it possible for others to replicate the model with the same dataset, or provide access to the model. In general. releasing code and data is often one good way to accomplish this, but reproducibility can also be provided via detailed instructions for how to replicate the results, access to a hosted model (e.g., in the case of a large language model), releasing of a model checkpoint, or other means that are appropriate to the research performed.
        \item While NeurIPS does not require releasing code, the conference does require all submissions to provide some reasonable avenue for reproducibility, which may depend on the nature of the contribution. For example
        \begin{enumerate}
            \item If the contribution is primarily a new algorithm, the paper should make it clear how to reproduce that algorithm.
            \item If the contribution is primarily a new model architecture, the paper should describe the architecture clearly and fully.
            \item If the contribution is a new model (e.g., a large language model), then there should either be a way to access this model for reproducing the results or a way to reproduce the model (e.g., with an open-source dataset or instructions for how to construct the dataset).
            \item We recognize that reproducibility may be tricky in some cases, in which case authors are welcome to describe the particular way they provide for reproducibility. In the case of closed-source models, it may be that access to the model is limited in some way (e.g., to registered users), but it should be possible for other researchers to have some path to reproducing or verifying the results.
        \end{enumerate}
    \end{itemize}

\item {\bf Open access to data and code}
    \item[] Question: Does the paper provide open access to the data and code, with sufficient instructions to faithfully reproduce the main experimental results, as described in supplemental material?
    \item[] Answer: \answerNA{} 
    \item[] Justification: This is a position paper. No code or datasets 
are introduced or used. 
    \item[] Guidelines:
    \begin{itemize}
        \item The answer \answerNA{} means that paper does not include experiments requiring code.
        \item Please see the NeurIPS code and data submission guidelines (\url{https://neurips.cc/public/guides/CodeSubmissionPolicy}) for more details.
        \item While we encourage the release of code and data, we understand that this might not be possible, so \answerNo{} is an acceptable answer. Papers cannot be rejected simply for not including code, unless this is central to the contribution (e.g., for a new open-source benchmark).
        \item The instructions should contain the exact command and environment needed to run to reproduce the results. See the NeurIPS code and data submission guidelines (\url{https://neurips.cc/public/guides/CodeSubmissionPolicy}) for more details.
        \item The authors should provide instructions on data access and preparation, including how to access the raw data, preprocessed data, intermediate data, and generated data, etc.
        \item The authors should provide scripts to reproduce all experimental results for the new proposed method and baselines. If only a subset of experiments are reproducible, they should state which ones are omitted from the script and why.
        \item At submission time, to preserve anonymity, the authors should release anonymized versions (if applicable).
        \item Providing as much information as possible in supplemental material (appended to the paper) is recommended, but including URLs to data and code is permitted.
    \end{itemize}

\item {\bf Experimental setting/details}
    \item[] Question: Does the paper specify all the training and test details (e.g., data splits, hyperparameters, how they were chosen, type of optimizer) necessary to understand the results?
    \item[] Answer: \answerNA{} 
    \item[] Justification: This is a position paper with no experiments 
or training procedures. 
    \item[] Guidelines:
    \begin{itemize}
        \item The answer \answerNA{} means that the paper does not include experiments.
        \item The experimental setting should be presented in the core of the paper to a level of detail that is necessary to appreciate the results and make sense of them.
        \item The full details can be provided either with the code, in appendix, or as supplemental material.
    \end{itemize}

\item {\bf Experiment statistical significance}
    \item[] Question: Does the paper report error bars suitably and correctly defined or other appropriate information about the statistical significance of the experiments?
    \item[] Answer: \answerNA{} 
    \item[] Justification: This is a position paper with no experiments. 
The probabilistic bounds in Section~3.2 are theoretical guarantees, 
not empirical measurements. 
    \item[] Guidelines:
    \begin{itemize}
        \item The answer \answerNA{} means that the paper does not include experiments.
        \item The authors should answer \answerYes{} if the results are accompanied by error bars, confidence intervals, or statistical significance tests, at least for the experiments that support the main claims of the paper.
        \item The factors of variability that the error bars are capturing should be clearly stated (for example, train/test split, initialization, random drawing of some parameter, or overall run with given experimental conditions).
        \item The method for calculating the error bars should be explained (closed form formula, call to a library function, bootstrap, etc.)
        \item The assumptions made should be given (e.g., Normally distributed errors).
        \item It should be clear whether the error bar is the standard deviation or the standard error of the mean.
        \item It is OK to report 1-sigma error bars, but one should state it. The authors should preferably report a 2-sigma error bar than state that they have a 96\% CI, if the hypothesis of Normality of errors is not verified.
        \item For asymmetric distributions, the authors should be careful not to show in tables or figures symmetric error bars that would yield results that are out of range (e.g., negative error rates).
        \item If error bars are reported in tables or plots, the authors should explain in the text how they were calculated and reference the corresponding figures or tables in the text.
    \end{itemize}

\item {\bf Experiments compute resources}
    \item[] Question: For each experiment, does the paper provide sufficient information on the computer resources (type of compute workers, memory, time of execution) needed to reproduce the experiments?
    \item[] Answer: \answerNA{} 
    \item[] Justification: This is a position paper with no computational 
experiments. 
    \item[] Guidelines:
    \begin{itemize}
        \item The answer \answerNA{} means that the paper does not include experiments.
        \item The paper should indicate the type of compute workers CPU or GPU, internal cluster, or cloud provider, including relevant memory and storage.
        \item The paper should provide the amount of compute required for each of the individual experimental runs as well as estimate the total compute. 
        \item The paper should disclose whether the full research project required more compute than the experiments reported in the paper (e.g., preliminary or failed experiments that didn't make it into the paper). 
    \end{itemize}
    
\item {\bf Code of ethics}
    \item[] Question: Does the research conducted in the paper conform, in every respect, with the NeurIPS Code of Ethics \url{https://neurips.cc/public/EthicsGuidelines}?
    \item[] Answer: \answerYes{} 
    \item[] Justification: The research conforms with the NeurIPS Code of 
Ethics. The paper advocates for safer deployment of LLM agents in 
safety-critical settings, raises no dual-use concerns, involves no 
human subjects, and introduces no artefacts with misuse potential. 
    \item[] Guidelines:
    \begin{itemize}
        \item The answer \answerNA{} means that the authors have not reviewed the NeurIPS Code of Ethics.
        \item If the authors answer \answerNo, they should explain the special circumstances that require a deviation from the Code of Ethics.
        \item The authors should make sure to preserve anonymity (e.g., if there is a special consideration due to laws or regulations in their jurisdiction).
    \end{itemize}

\item {\bf Broader impacts}
    \item[] Question: Does the paper discuss both potential positive societal impacts and negative societal impacts of the work performed?
    \item[] Answer: \answerYes{} 
    \item[] Justification: The paper's primary contribution is a safety 
architecture for LLM agents in safety-critical deployments. The 
positive societal impact is improved runtime assurance for autonomous 
systems operating near humans. Regarding negative impacts: the 
framework constrains rather than enables harmful behaviour, introduces 
no new attack surface, and the open problems identified in Section~4 
are prerequisites for deployment rather than risks introduced by the 
work itself. One potential concern is that premature deployment of the 
architecture without closing the open problems of Section~4 could 
create false assurance; this is explicitly cautioned against in 
Section~4 and the Conclusion. 
    \item[] Guidelines:
    \begin{itemize}
        \item The answer \answerNA{} means that there is no societal impact of the work performed.
        \item If the authors answer \answerNA{} or \answerNo, they should explain why their work has no societal impact or why the paper does not address societal impact.
        \item Examples of negative societal impacts include potential malicious or unintended uses (e.g., disinformation, generating fake profiles, surveillance), fairness considerations (e.g., deployment of technologies that could make decisions that unfairly impact specific groups), privacy considerations, and security considerations.
        \item The conference expects that many papers will be foundational research and not tied to particular applications, let alone deployments. However, if there is a direct path to any negative applications, the authors should point it out. For example, it is legitimate to point out that an improvement in the quality of generative models could be used to generate Deepfakes for disinformation. On the other hand, it is not needed to point out that a generic algorithm for optimizing neural networks could enable people to train models that generate Deepfakes faster.
        \item The authors should consider possible harms that could arise when the technology is being used as intended and functioning correctly, harms that could arise when the technology is being used as intended but gives incorrect results, and harms following from (intentional or unintentional) misuse of the technology.
        \item If there are negative societal impacts, the authors could also discuss possible mitigation strategies (e.g., gated release of models, providing defenses in addition to attacks, mechanisms for monitoring misuse, mechanisms to monitor how a system learns from feedback over time, improving the efficiency and accessibility of ML).
    \end{itemize}
    
\item {\bf Safeguards}
    \item[] Question: Does the paper describe safeguards that have been put in place for responsible release of data or models that have a high risk for misuse (e.g., pre-trained language models, image generators, or scraped datasets)?
    \item[] Answer: \answerNA{} 
    \item[] Justification: The paper releases no models, datasets, or 
code and proposes no artefacts with misuse risk. 
    \item[] Guidelines:
    \begin{itemize}
        \item The answer \answerNA{} means that the paper poses no such risks.
        \item Released models that have a high risk for misuse or dual-use should be released with necessary safeguards to allow for controlled use of the model, for example by requiring that users adhere to usage guidelines or restrictions to access the model or implementing safety filters. 
        \item Datasets that have been scraped from the Internet could pose safety risks. The authors should describe how they avoided releasing unsafe images.
        \item We recognize that providing effective safeguards is challenging, and many papers do not require this, but we encourage authors to take this into account and make a best faith effort.
    \end{itemize}

\item {\bf Licenses for existing assets}
    \item[] Question: Are the creators or original owners of assets (e.g., code, data, models), used in the paper, properly credited and are the license and terms of use explicitly mentioned and properly respected?
    \item[] Answer: \answerYes{} 
    \item[] Justification: All third-party works cited are properly 
referenced with full bibliographic entries. The paper uses no datasets 
or software assets directly; all cited benchmarks and tools are 
referenced for motivational purposes only, not reused or redistributed. 
    \item[] Guidelines:
    \begin{itemize}
        \item The answer \answerNA{} means that the paper does not use existing assets.
        \item The authors should cite the original paper that produced the code package or dataset.
        \item The authors should state which version of the asset is used and, if possible, include a URL.
        \item The name of the license (e.g., CC-BY 4.0) should be included for each asset.
        \item For scraped data from a particular source (e.g., website), the copyright and terms of service of that source should be provided.
        \item If assets are released, the license, copyright information, and terms of use in the package should be provided. For popular datasets, \url{paperswithcode.com/datasets} has curated licenses for some datasets. Their licensing guide can help determine the license of a dataset.
        \item For existing datasets that are re-packaged, both the original license and the license of the derived asset (if it has changed) should be provided.
        \item If this information is not available online, the authors are encouraged to reach out to the asset's creators.
    \end{itemize}

\item {\bf New assets}
    \item[] Question: Are new assets introduced in the paper well documented and is the documentation provided alongside the assets?
    \item[] Answer: \answerNA{} 
    \item[] Justification: The paper introduces no new datasets, models, 
software, or other artefacts requiring documentation. 
    \item[] Guidelines:
    \begin{itemize}
        \item The answer \answerNA{} means that the paper does not release new assets.
        \item Researchers should communicate the details of the dataset\slash code\slash model as part of their submissions via structured templates. This includes details about training, license, limitations, etc. 
        \item The paper should discuss whether and how consent was obtained from people whose asset is used.
        \item At submission time, remember to anonymize your assets (if applicable). You can either create an anonymized URL or include an anonymized zip file.
    \end{itemize}

\item {\bf Crowdsourcing and research with human subjects}
    \item[] Question: For crowdsourcing experiments and research with human subjects, does the paper include the full text of instructions given to participants and screenshots, if applicable, as well as details about compensation (if any)? 
    \item[] Answer: \answerNA{} 
    \item[] Justification: The paper involves no crowdsourcing or research 
with human subjects. 
    \item[] Guidelines:
    \begin{itemize}
        \item The answer \answerNA{} means that the paper does not involve crowdsourcing nor research with human subjects.
        \item Including this information in the supplemental material is fine, but if the main contribution of the paper involves human subjects, then as much detail as possible should be included in the main paper. 
        \item According to the NeurIPS Code of Ethics, workers involved in data collection, curation, or other labor should be paid at least the minimum wage in the country of the data collector. 
    \end{itemize}

\item {\bf Institutional review board (IRB) approvals or equivalent for research with human subjects}
    \item[] Question: Does the paper describe potential risks incurred by study participants, whether such risks were disclosed to the subjects, and whether Institutional Review Board (IRB) approvals (or an equivalent approval/review based on the requirements of your country or institution) were obtained?
    \item[] Answer: \answerNA{} 
    \item[] Justification: The paper involves no human subjects research 
and therefore requires no IRB approval or equivalent. 
    \item[] Guidelines:
    \begin{itemize}
        \item The answer \answerNA{} means that the paper does not involve crowdsourcing nor research with human subjects.
        \item Depending on the country in which research is conducted, IRB approval (or equivalent) may be required for any human subjects research. If you obtained IRB approval, you should clearly state this in the paper. 
        \item We recognize that the procedures for this may vary significantly between institutions and locations, and we expect authors to adhere to the NeurIPS Code of Ethics and the guidelines for their institution. 
        \item For initial submissions, do not include any information that would break anonymity (if applicable), such as the institution conducting the review.
    \end{itemize}

\item {\bf Declaration of LLM usage}
    \item[] Question: Does the paper describe the usage of LLMs if it is an important, original, or non-standard component of the core methods in this research? Note that if the LLM is used only for writing, editing, or formatting purposes and does \emph{not} impact the core methodology, scientific rigor, or originality of the research, declaration is not required.
    \item[] Answer: \answerNA{} 
    \item[] Justification: LLMs are the object of study in this paper, 
not a component of the research methodology. No LLMs were used as 
an important, original, or non-standard component of the core methods, 
proofs, or architectural proposals. Any use of LLMs was limited to 
writing assistance, which does not require declaration under NeurIPS 
2026 policy. 
    \item[] Guidelines:
    \begin{itemize}
        \item The answer \answerNA{} means that the core method development in this research does not involve LLMs as any important, original, or non-standard components.
        \item Please refer to our LLM policy in the NeurIPS handbook for what should or should not be described.
    \end{itemize}

\end{enumerate}

\end{document}